\documentclass{article}

\usepackage{amsmath,amsfonts,bm}

\newcommand{\all}{\calK}

\newcommand{\out}{\mW_O}
\newcommand{\val}{\mW_V}
\newcommand{\key}{\mW_K}
\newcommand{\query}{\mW_Q}
\newcommand{\weights}{\mW} %
\newcommand{\bias}{\vb}
\newcommand{\ReLU}{{\mathrm{ReLU}}}
\newcommand{\SiLU}{{\mathrm{SiLU}}}
\newcommand{\relu}{{\sigma_R}}
\newcommand{\attn}{{\mathrm{Attn}}}
\newcommand{\ff}{{\mathrm{FF}}}
\newcommand{\tf}{{\mathrm{TF}}}
\newcommand{\LN}{{\mathrm{RMSN}}}
\newcommand{\ie}{\textit{i}.\textit{e}.}
\newcommand{\st}{\textit{s}.\textit{t}.}
\newcommand{\eg}{\textit{e}.\textit{g}.}
\newcommand{\wrt}{\textit{w}.\textit{r}.\textit{t}.}
\newcommand{\calL}{{\mathcal{L}}}
\newcommand{\N}{\mathbb{N}}
\let\tn\textnormal
\newcommand{\bin}{\tn{bin}\hspace{1.5pt}}

\newcommand{\norm}[1]{\left \lVert #1 \right \rVert}

\newcommand{\calT}{{\mathrm{step}}}
\newcommand{\bB}{\bm{\mathcal{B}}}

\newcommand{\calK}{{\mathcal{K}}}
\newcommand{\calS}{{\mathcal{S}}}

\newcommand*{\zero}{{\bm 0}}

\def\eqref#1{equation~\ref{#1}}

\def\twopartref#1\val\val#2{parts \ref{#1} and \ref{#2}}

\def\1{\boldsymbol{1}}

\def\eps{{\epsilon}}

\def\rmX{{\mathbf{X}}}

\def\va{{\bm{a}}}
\def\vb{{\bm{b}}}

\def\vs{{\bm{s}}}

\def\vu{{\bm{u}}}
\def\vv{{\bm{v}}}

\def\vx{{\bm{x}}}
\def\vy{{\bm{y}}}
\def\vz{{\bm{z}}}

\def\mI{{\bm{I}}}

\def\mL{{\bm{L}}}

\def\mP{{\bm{P}}}

\def\mW{{\bm{W}}}
\def\mX{{\bm{X}}}

\def\mZ{{\bm{Z}}}

\DeclareMathAlphabet{\mathsfit}{\encodingdefault}{\sfdefault}{m}{sl}
\SetMathAlphabet{\mathsfit}{bold}{\encodingdefault}{\sfdefault}{bx}{n}

\def\sA{{\mathbb{A}}}

\def\sL{{\mathbb{L}}}

\newcommand{\R}{\mathbb{R}}
\newcommand{\Z}{\mathbb{Z}}

\newcommand{\softmax}{\sigma_s}

\DeclareMathOperator*{\argmax}{arg\,max}
\DeclareMathOperator*{\argmin}{arg\,min}

\usepackage{microtype}
\usepackage{graphicx}
\usepackage{subfigure}
\usepackage{booktabs} 
\usepackage{booktabs}
\usepackage{multirow}
\usepackage{enumitem}
\usepackage{hhline}

\usepackage{hyperref}

\usepackage[accepted]{icml2025}

\usepackage{amsmath}
\usepackage{amssymb}
\usepackage{mathtools}
\usepackage{amsthm}
\usepackage{thmtools}
\usepackage{thm-restate}

\usepackage[capitalize,noabbrev]{cleveref}

\theoremstyle{plain}
\newtheorem{theorem}{Theorem}[section]
\newtheorem{proposition}[theorem]{Proposition}
\newtheorem{lemma}[theorem]{Lemma}
\newtheorem{corollary}[theorem]{Corollary}
\theoremstyle{definition}
\newtheorem{definition}[theorem]{Definition}

\theoremstyle{remark}

\usepackage[textsize=tiny]{todonotes}

\icmltitlerunning{On Expressive Power of Looped Transformers}

\begin{document}

\twocolumn[
\icmltitle{On Expressive Power of Looped Transformers:\\ Theoretical Analysis and Enhancement via Timestep Encoding}



\icmlsetsymbol{equal}{*}

\begin{icmlauthorlist}
\icmlauthor{Kevin Xu}{utokyo}
\icmlauthor{Issei Sato}{utokyo}
\end{icmlauthorlist}

\icmlaffiliation{utokyo}{The University of Tokyo}

\icmlcorrespondingauthor{Kevin Xu}{kevinxu@g.ecc.u-tokyo.ac.jp}
\icmlcorrespondingauthor{Issei Sato}{sato@g.ecc.u-tokyo.ac.jp}

\icmlkeywords{Machine Learning, ICML}

\vskip 0.3in
]



\printAffiliationsAndNotice{}  

\begin{abstract}
Looped Transformers provide advantages in parameter efficiency, computational capabilities, and generalization for reasoning tasks. However, their expressive power regarding function approximation remains underexplored. In this paper, we establish the approximation rate of Looped Transformers by defining the modulus of continuity for sequence-to-sequence functions. This reveals a limitation specific to the looped architecture. That is, the analysis prompts the incorporation of scaling parameters for each loop, conditioned on timestep encodings. Experiments validate the theoretical results, showing that increasing the number of loops enhances performance, with further gains achieved through the timestep encoding. Code is available at~\url{https://github.com/kevin671/tmlt}.
\end{abstract}

\section{Introduction}
Transformers~\cite{Vaswani2017AttentionIA} have become the standard architecture for a wide range of machine learning tasks, including natural language processing and computer vision. However, they exhibit certain limitations, particularly when applied to complex tasks. The expressive power of Transformers is theoretically constrained~\cite{merrill2023parallelism,feng2023towards}, and they empirically struggle with reasoning and planning problems~\cite{kambhampati2024position}. Although chain-of-thought reasoning~\cite{wei2022chain} can mitigate these challenges in some cases, it typically relies on manually crafted prompts or costly intermediate supervision. Moreover, Transformers encounter difficulties with length generalization~\cite{deletang2023neural} and require substantial computational resources as the number of model parameters increases~\cite{pope2022efficiently}. 

To address these limitations, Looped Transformers presents a promising approach. The architecture consists of fixed-size Transformer layers, in which the output is recursively fed back into the input.
Looped Transformers exhibit advantages in parameter efficiency thanks to their weight-tying structure~\cite{Lan2020ALBERT,takase2021lessons,csordas2024moeut,bae2024relaxed}, achieving performance comparable to standard Transformers while using fewer parameters.
Additionally, they are well suited for size generalization by adjusting the loop count based on task complexity~\cite{dehghani2018universal,fan2024looped}. Their recursive structure endows them with the expressive power to emulate iterative algorithms and universal computational capabilities, akin to programmable computers~\cite{Giannou2023LoopedTA}. Furthermore, their inductive bias enhances performance on reasoning tasks~\cite{saunshi2025reasoning}.

In contrast, the expressive power of Looped Transformers and the properties unique to the looped architecture in function approximation remain unexplored.
The expressive power of standard Transformers, on the other hand, has been examined extensively in prior studies. These studies show that Transformers can be universal approximators for continuous permutation-equivariant functions on compact domains~\cite{Yun2020Are, kajitsuka2024are}. Furthermore, their approximation rate has been analyzed by identifying specific properties of the target functions~\cite{takakura23a, jiang2024, wang2024understanding}, providing insights into the underlying characteristics of Transformer architectures.
However, these findings cannot be directly extended due to the weight-tying constraints.
Although the approximation rate of looped ReLU networks has been established only recently~\cite{zhang23}, that of Looped Transformers remains unknown.

Our contributions are summarized as follows:
\begin{itemize}[leftmargin=*]
    \item We establish the approximation rate of Looped Transformers for fixed-length continuous sequence-to-sequence functions by introducing three newly defined types of modulus of continuity.
    \item We identify an intrinsic limitation of the looped architecture and address it by introducing timestep encoding and the Timestep-Modulated Looped Transformer (TMLT).
\end{itemize}

\section{Background}
This section defines the Transformer and Looped Transformer architectures, reviews related work, and examines prior theoretical studies on the function approximation capabilities of Transformers and weight-tied networks, thereby clarifying the research question addressed in this paper.

\textbf{Notations:} Vectors are represented by lowercase boldface letters \eg, $\vv$, and matrices are denoted by uppercase boldface letters \eg, $\mX$. The $i$-th element of a vector $\vv$ is denoted by $\vv_i$, and the $(i,j)$-th element of a matrix $\mX$ is denoted by $\mX_{i,j}$. The $n$-th column of a matrix $\mX$ is denoted by $\mX_{:,n}$. 

Given an input sequence \( \mX = [\vx_1, \vx_2, \dots, \vx_N] \in \mathbb{R}^{m \times N} \), where \( \vx_i \in \mathbb{R}^m \), and a function \( f: \mathbb{R}^m \to \mathbb{R}^m \), the token-wise application of \( f \) is denoted by the bold symbol $\bm{f}$ \ie
\begin{equation*}
\bm{f}(\mX) = [f(\vx_1), f(\vx_2), \dots, f(\vx_N)] \in \mathbb{R}^{m \times N}.
\end{equation*}
For \( p \in [1, \infty) \), the \( p \)-norm, denoted by \( \norm{\cdot}_p \), represents the entry-wise \( p \)-norm. This norm applies to both vectors and matrices \eg, \( \norm{\mX}_p \). The $L^p$-norm of a function is defined for $p \in [1, \infty)$ as:
\begin{equation*}
    \| f \|_{L^p} \coloneq \Big (\int_{\Omega} \norm{f(\mX)}_p^p d\mX \Big )^{1/p},
\end{equation*}
where \(\Omega\) represents the domain of the function \(f\).

\subsection{Transformer Architecture}
Given an input sequence $\mX \in \R^{m \times N}$, composed of $N$ token embedding of dimension size $m$, the self-attention layers with $h$ heads and head size $s$, and the feed-forward layer with width size $q$, are defined as follows:
\begin{align*}
    &\attn({\mX}) = \sum_{i=1}^h \out^i \val^i \mX\softmax\left[(\key^{i}\mX)^\top\query^i\mX\right], \\
    &\ff(\mX_{:,n}) = \weights_2\relu(\weights_1\mX_{:,n} + \bias_1) + \bias_2,
\end{align*}
where $\out^i \in \R^{m \times s}, \val^i,\, \key^{i},\, \query^i \in \R^{s \times m}, \weights_1 \in \R^{q \times m}, \weights_2 \in \R^{m \times q}, \bias_1 \in \R^q, \bias_2 \in \R^m$ are parameters, $\relu$ denotes $\ReLU$ function, and $\softmax$ denotes a softmax operator applied to the columns of the matrix. 

Transformer block $\tf: \R^{m \times N} \to \R^{m \times N}$ is defined by
\begin{align*}
    \mX' &= \mX + \attn(\mX),\\
    \tf(\mX) &= \mX' + \bm{\ff}(\mX'),
\end{align*}
where $\bm{\ff}$ represents token-wise $\ff$. In other words,
\begin{equation*}
    \tf = (\mathrm{id} + \bm{\ff}) \circ (\mathrm{id} + \attn),
\end{equation*}
where $\mathrm{id}$ denotes the identity mapping, where we omit the domain of definition for simplicity.
For the analysis of expressive power in \Cref{sec:unversalapproximators}, we exclude layer normalization and our constructive proof relies on the softmax function to approximate the hardmax function as in previous studies~\citep{Yun2020Are,kim2023provable} 

\subsection{Looped Transformer}
Looped Transformer with a single layer is represented as:
\begin{equation*}
    \bm{\mathcal{L}_}{2} \circ \tf^{\circ r} \circ \bm{\mathcal{L}}_{1},
\end{equation*}
where $\bm{\mathcal{L}}_2$ and $\bm{\mathcal{L}}_1$ represent token-wise affine linear layers, and $\tf^{\circ r}$ denotes the composition of $\tf$ applied $r$ times. While we focus on single-layer as \cite{dehghani2018universal,Lan2020ALBERT,yang2024looped,fan2024looped}, they can also be implemented with multiple layers as \cite{csordas2024moeut,bae2024relaxed,saunshi2025reasoning}.

\paragraph{Overview of Previous Work}
The recursive structure was introduced into Transformers~\cite{dehghani2018universal}, where the number of loops can be adaptively adjusted, allowing for size generalization~\cite{fan2024length}.
Looped Transformers are closely related to weight-tying Transformers~\cite{Lan2020ALBERT,takase2021lessons}, achieving performance comparable to standard Transformers using fewer parameters.
Deep equilibrium models~\cite{bai2019deep}, which compute fixed points of iterative layers, are also related.
In addition, the recursive structure enables the model to emulate iterative algorithms, including basic computational primitives \cite{Giannou2023LoopedTA} and learning algorithms \cite{giannou2024newtons,yang2024looped}. 
Furthermore, recent studies have demonstrated that Looped Transformers exhibit an inductive bias towards reasoning tasks~\cite{saunshi2025reasoning}. To improve performance, more sophisticated architectures, such as mixture-of-experts~\cite{csordas2024moeut} and relaxed weight-tying~\cite{bae2024relaxed}, have been introduced.

%

\subsection{Theoretical Analysis on Expressive Power}
We review related work and summarize the comparisons between our problem setting and previous studies in Table~\ref{tbl:comparison}.

\begin{table*}[t]
\caption{Comparisons of our problem setting with related work on the theoretical analysis of function approximation.}
\label{tbl:comparison}
\vskip 0.15in
\begin{center}
\begin{small}
\begin{tabular}{c|cccc}
\toprule
Paper & Model Type & Function Class & Approximation Rate & Looped (Weight-Tying) \\ 
\hhline{=|====}
\citet{yarotsky2018optimal} & FFN & Continuous functions & $\checkmark$ & $\times$ \\
\hline
\citet{Yun2020Are} & Transformer & Continuous seq-to-seq functions & $\times$ & $\times$ \\
\hline
\citet{takakura23a} & Transformer & $\gamma$-smooth infinite-length & $\checkmark$ & $\times$ \\
\hline
\citet{kajitsuka2024are} & Transformer & Continuous seq-to-seq functions & $\times$ & $\times$ \\
\hline
\citet{jiang2024} & Transformer & Temporal coupled functions & $\checkmark$ & $\times$ \\
\hline
\citet{wang2024understanding} & Transformer & Long but sparse memories & $\checkmark$ & $\times$ \\
\hline
\citet{zhang23} & FFN & Continuous functions & $\checkmark$ & $\checkmark$ \\
\hline
\textbf{Ours} & Transformer & Continuous seq-to-seq functions & $\checkmark$  & $\checkmark$ \\
\bottomrule
\end{tabular}
\end{small}
\end{center}
\vskip -0.1in
\end{table*}
\paragraph{Universality of Transformers}
The universal approximation theorem for fully connected neural networks~\cite{cybenko1989, hornik1989multilayer} shows that networks of sufficient size can approximate certain classes of functions with arbitrarily low error. 
For Transformers, the target function class extends to sequence-to-sequence functions. Transformers compute a {\em contextual mapping} of the input, which requires capturing the entire sequence and computing the token embedding within context~\cite{Yun2020Are}, formulated as:
\begin{definition}[\citealp{Yun2020Are}] \label{def:context-mapping}
Consider a finite set $\sL \subset \R^{d \times N}$.
A map $\mathrm{CM}: \sL \to \R^{1 \times N}$ defines a {\em contextual mapping} if the map satisfies the following:
\begin{enumerate}[leftmargin=*, nosep]
    \item For any $\mL \in \sL$, the $N$ entries in $\mathrm{CM}(\mL)$ are all distinct.
    \item For any $\mL, \mL' \in \sL$, with $\mL \neq \mL'$, all entries of $\mathrm{CM}(\mL)$ and $\mathrm{CM}(\mL')$ are distinct.
\end{enumerate}
\end{definition}
Prior studies have shown that Transformers can compute contextual mappings, enabling memorization~\cite{kim2023provable} and universal approximation~\cite{Yun2020Are, kajitsuka2024are}.

For Looped Transformers, as the fixed parameters of a single Transformer layer are used, the results of previous studies cannot be directly applied. This leads to the question: \textit{Can Looped Transformers compute contextual mappings?} and \textit{Are they universal approximators?}

\paragraph{Approximation Rate of Transformers}
Beyond the universality, the approximation rate provides deeper insights into the characteristics of models \cite{barron1993universal, yarotsky2018optimal}. This rate is derived as an upper bound of error in terms of the properties of the target functions and the complexity of the networks. For Transformers, recent studies have investigated these rates and the nature of the target functions~\cite{takakura23a,jiang2024,wang2024understanding}. Specifically, they have shown conditions under which Transformers can overcome the curse of dimensionality~\cite{takakura23a} and revealed structures in target functions that Transformers can effectively approximate~\cite{jiang2024,wang2024understanding}.

Our study focuses on understanding the architectural properties of Looped Transformers, particularly in comparison to standard Transformers. To this end, we explore the approximation rate and investigate the properties of target functions that determine their approximation errors.
%

\paragraph{Expressive Power of Weight-Tied Neural Networks} 
Recently, it has been shown that single \emph{fixed-size} networks can serve as universal approximators in a parameter-efficient manner; that is, the parameter count depends solely on the input dimension, not the approximation error~\cite{zhang23}. 
Furthermore, the approximation rate of weight-tied ReLU networks has been established with respect to the number of loops and the modulus of continuity of continuous functions~\cite{zhang23}. The modulus of continuity for \( g: \mathbb{R}^d \to \mathbb{R} \) and \( \delta \geq 0 \) is defined as:
\begin{equation*}\label{eq:modul_g}
    \omega_{g}(\delta) \coloneq \sup \big\{ \lvert g(\vx) - g(\vx') \rvert : \|\vx - \vx'\|_2 \leq \delta\big\}.
\end{equation*}
Our question is whether the results can be extended to sequence-to-sequence functions and Transformers, which require contextual mappings. For a sequence-to-sequence function $f:\R^{d \times N} \to \R^{d \times N}$, the modulus of continuity can be generalized as:
\begin{equation*}\label{eq:modul_f}
\omega_{f}(\delta) \coloneqq \sup\big\{ \|f(\mX) - f(\mX')\|_p : \|\mX - \mX'\|_2 \le \delta \big\}
\end{equation*}
We investigate whether this modulus of continuity alone can determine the approximation rate. 

For Looped Transformers, it has been shown that they can represent standard Transformers, although their parameter count grows with both the desired approximation accuracy and the sequence length~\cite{saunshi2025reasoning}. Moreover, no existing work has established their approximation rate. 

\section{Approximation Rate Analysis}\label{sec:unversalapproximators}
In this section, we establish the approximation rate of Looped Transformers. We define three types for the modulus of continuity in~\Cref{subsec:continuity} that determine the approximation rate. The main results are presented in~\Cref{subsec:mainresult}, followed by a proof sketch in \Cref{sec:proof_sketch}.

\subsection{Preliminaries}
The target function class of our analysis is continuous functions that Transformers can represent. Specifically, these are {\em permutation-equivariant} functions, defined as follows:
\begin{definition}[\citealp{Yun2020Are}] 
A function $f : \R^{d \times N} \to \R^{d \times N}$ is said to be {\em permutation equivariant} if $f(\mX \mP) = f(\mX) \mP$ holds for any permutation matrix $\mP$. Let $\mathcal{F}_{\mathrm{PE}}(\Omega)$ denote the set of continuous functions, defined on $\Omega$, that are permutation equivariant.
\end{definition}

We evaluate both the number of parameters and the \textit{bit complexity}, the maximum number of bits required to represent the network's weights~\cite{vardi2022on,kim2023provable}. 

In our proofs, we introduce \textit{IDs} for tokens, sequences, and tokens within sequences as theoretical constructs to formalize contextual mappings.
\begin{definition} 
A \textit{token ID} is a unique integer assigned to each token. A \textit{sequence ID} uniquely identifies each sentence. A \textit{contextual token ID} uniquely identifies a specific token within a specific sentence. We denote the set of contextual token IDs as $\mathcal{K} = {0, 1, \dots, K-1}$, with corresponding embeddings $\vy_k \in \mathbb{R}^d$ for each $k \in \mathcal{K}$. 
\end{definition}
This notion is defined in \citet{kim2023provable}, to which we refer for further details, for constructive proofs of contextual mappings. The actual construction of contextual token IDs may vary depending on the specific proof. In our case, we adopt a different construction from that of \citet{kim2023provable}.

\subsection{Definition of modulus of Continuity}
\label{subsec:continuity}
As briefly mentioned in the preliminary discussion, we define the modulus of continuity in Eq.~\ref{eq:modul_f} as:
\begin{definition}[Modulus of Sequence Continuity]
\label{def:sentence_continuity}
Given a sequence-to-sequence continuous function $f:\R^{d \times N} \to \R^{d \times N}$, the modulus of \emph{sequence continuity} is defined by:
\begin{equation*}
    \omega_{f}(\delta) \coloneqq \sup\big\{ \|f(\mX) - f(\mX')\|_p : \|\mX - \mX'\|_2 \le \delta \big\}.
\end{equation*}
\end{definition}
We omit the subscript $p$ for simplicity. This quantifies how the output sequence shifts relative to differences in input, hence referred to as \textit{sequence continuity}.

We found that this alone is insufficient to determine the approximation rate of Looped Transformers, in contrast to the case of ReLU networks~\cite{zhang23}. Informally, this issue arises because Transformers compute contextual mappings. 
We notably identified two additional types of modulus of continuity, defined as follows.
\begin{definition}[Modulus of Contextual Continuity]
Given a sequence-to-sequence continuous function $f:\R^{d \times N} \to \R^{d \times N}$, the modulus of \emph{contextual continuity} is defined by:
\begin{equation*}
\begin{split}
    \omega^{\text{cont}}_f(\delta)\coloneqq &\sup_{n, \mX, \mX'}\big\{\|f(\mX)_{:,n}-f(\mX')_{:,n}\|_p \\ 
    &\quad :\|\mX-\mX'\|_2\le \delta,\ \mX_{:,n} = \mX'_{:,n} \big\}.
\end{split}
\end{equation*}
\end{definition}
\begin{definition}[Modulus of Token Continuity]
Given a sequence-to-sequence continuous function $f:\R^{d \times N} \to \R^{d \times N}$, the modulus of \emph{token continuity} is defined by: 
\begin{equation*}
\begin{split}
    &\omega^{\text{tok}}_f(\delta)\coloneqq \sup_{n, \mX, \mX'}\big\{\|f(\mX)_{:,n}-f(\mX')_{:,n}\|_p: \\
    &\|\mX_{:,n}-\mX'_{:,n}\|_2\le \delta, \ \mX_{:,q} = \mX'_{:,q} \ (\forall q \neq n)\big\}.
\end{split}
\end{equation*}
\end{definition}
The \emph{modulus of contextual continuity} quantifies the variation in the contextual token embeddings induced by perturbations of context. For example, consider the sentences: (1) ``I \underline{write} papers'' and (2) ``You \underline{write} books''. It measures the difference in the contextual token embedding of the same word `\underline{write}' within different contexts.

On the other hand, the \emph{modulus of token continuity} quantifies the variation in the output embedding caused by perturbations to the token itself within the same context such as (1) ``I \underline{write} papers'' and (2) ``I \underline{draft} papers''.

\subsection{Main Result}\label{subsec:mainresult}
The result establishes the approximation rate of Looped Transformers in terms of the number of loops and the three types of moduli of continuity of the target function.
\begin{restatable}{theorem}{universal}
\label{thm:universal}
Given a function $f \in \mathcal{F}_{\mathrm{PE}}([0,1]^{d \times N})$, $r > N$, there exists a Looped Transformer, composed of $\mathrm{TF}:\R^{(17d+9)\times N}\to\R^{(17d+9)\times N}$ with two heads, head size $1$, and width size of $q=49d+25$, and two affine linear maps $\mathcal{L}_1:\R^d\to\R^{17d+9}$ and $\mathcal{L}_2:\R^{17d+9}\to\R^d$ \st
\begin{equation*}
\begin{split}
    &\big\|\bm{\mathcal{L}}_2\circ \mathrm{TF}^{\circ r}\circ \bm{\mathcal{L}}_1-f\big\|_{L^p} \\
    &\leq (Nd)^{\frac{1}{p}} \Big(\omega^{\textnormal{tok}}_f(\delta\sqrt{d}) + \omega^{\textnormal{cont}}_f(\delta\sqrt{Nd}) \Big) + \omega_f(\delta \sqrt{Nd}) \\ & \qquad \qquad+ \mathcal{O}(N^{2/p} \delta^{d/p}) + \mathcal{O}\left(\left( (M\delta)^{-1} dN \right)^{1/p}\right), \\ & \qquad  \text{where } \delta = \big((r-N)/2\big)^{-1/((N+1)d+1)},
\end{split}
\end{equation*}
where $M$ is the maximum absolute value of the model parameters, and the bit complexity is $\mathcal{O}(\delta^{-(N+1)d})$.
\end{restatable}
\Cref{thm:universal} shows that increasing the number of loops $r$ reduces the approximation error. Under infinite-precision weights, this leads to a universal approximation theorem.
\begin{corollary}[Universality]
The hypothesis space of Looped Transformers, defined by
\begin{equation*}
\begin{split}
    \mathcal{H} \coloneqq &\big\{\bm{\mathcal{L}_2} \circ \tf^{\circ r} \circ \bm{\mathcal{L}_1}:[0,1]^{d \times N} \to [0,1]^{d \times N} \mid \\ & \ m, q\le Cd,\ h=2, \ s=1, \ r \in \mathbb{N},\, \mW \in \mathbb{R}^{n_w} \big\},
\end{split}
\end{equation*}
where $C$ is a positive constant, \(\mW\) denotes the flattened set of all weights in the network, and \(n_w\) represents the total number of these weights,
is dense in $\mathcal{F}_{\mathrm{PE}}([0,1]^{d \times N})$ \wrt the $L^p$ norm.
\end{corollary}

This approximation analysis highlights the characteristics of Looped Transformers, including both their capabilities and limitations, as summarized below:
\begin{itemize}[leftmargin=*]
    \item While the number of parameters remains fixed at $O(d)$, independent of the desired approximation accuracy and the sequence length, the error can be reduced by increasing the number of loops.
    \item Looped Transformers, even with weight-tied self-attention using a hard-max function, can compute contextual mappings and become universal approximators.
    \item The approximation rate depends on three types of continuity, with contextual and token dependencies unique to Looped Transformers; these dependencies are not present in standard Transformers or looped ReLU networks.
\end{itemize}
Our contribution lies in establishing the approximation rate with respect to the number of loops, based on novel moduli of continuity that are unique to Looped Transformers.

Furthermore, the additional dependency can amplify the approximation error, revealing a limitation inherent to Looped Transformers. A detailed discussion of this issue, along with improvement methods, is provided in \Cref{sec:lim}.

\subsection{Proof Sketch}\label{sec:proof_sketch}
This section presents a proof sketch, emphasizing distinctions from prior studies and challenges unique to the looped architecture. The formal proof is provided in~\Cref{appendix:function}.

The basic strategy involves approximating the continuous target function $f$ with a piecewise constant function $\bar{f}$, which is approximated by the network, denoted by $\tilde{f}$. For $\delta^{-1} \in \mathbb{N},\ \delta^{-1}\ge2$, the input space $[0,1]^{d \times N}$ is divided into discretized cubes with width $\delta$, denoted by $\{Q_{\bB}\}_{\bB \in \{0,1,\dotsc,\delta^{-1} - 1\}^{d\times N}}$. Each cube is assigned a representative point $\hat{\mX}_{\bB} \in Q_{\bB}$, and the piecewise constant function $\bar{f}$ is then defined as:
\begin{equation}\label{eq:bB}
    \bar{f}(\mX) = f(\hat{\mX}_{\bB}) \quad \text{where } \bB \text{ satisfies } \mX \in Q_{\bB}.
\end{equation}
The approximation with networks consists of three steps. First, the network assigns a \textit{token ID} to each token. Second, it assigns a \textit{sequence ID}. The combination of the \textit{token ID} and \textit{sequence ID} constitutes the \textit{contextual token IDs} as in Fig.~\ref{fig:tok-seq}. Finally, these are mapped to embeddings that represent the output of the target function at each token.

\begin{figure}[t]
\begin{center}
\centerline{\includegraphics[width=\columnwidth]{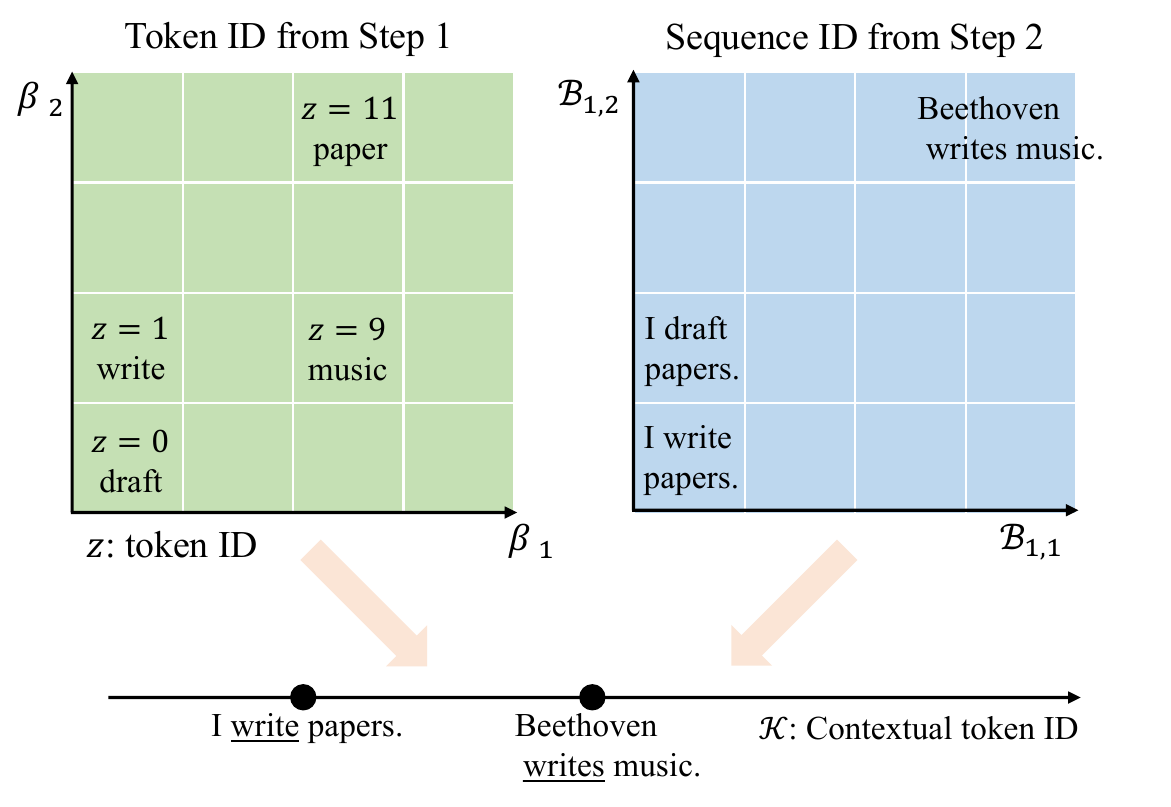}}
\vskip -0.1in
\caption{The networks construct contextual token IDs by combining token IDs with sequence IDs.}
\label{fig:tok-seq}
\end{center}
\vskip -0.2in
\end{figure}

\begin{figure*}[t]
\begin{center}
    \includegraphics[width=\linewidth]{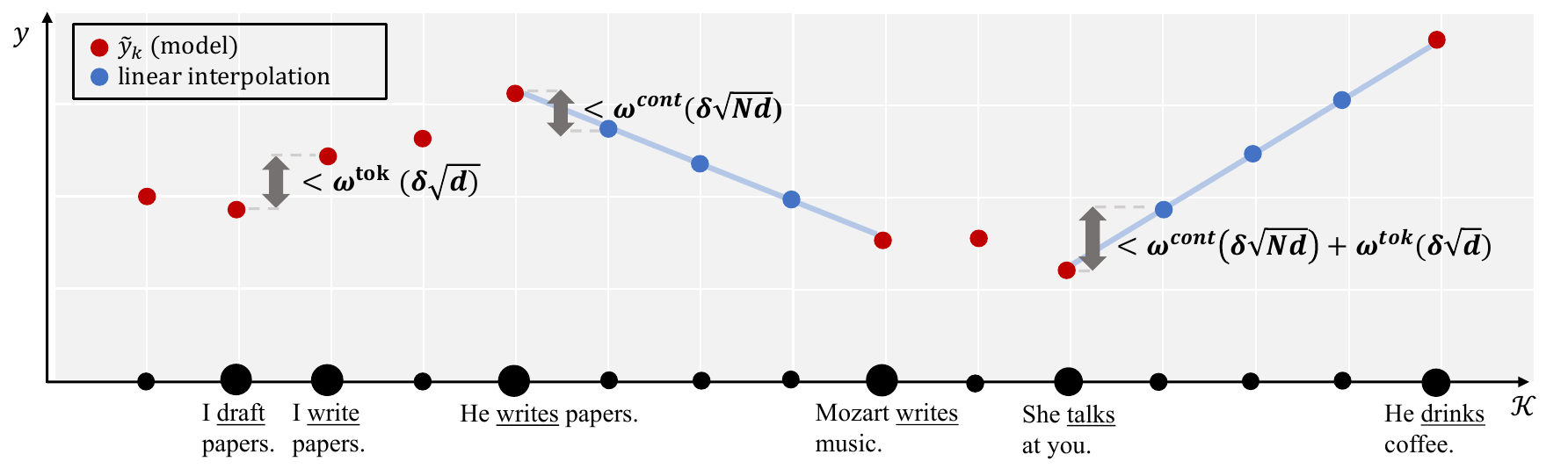}
\end{center}
\vskip -0.2in
\caption{Approximation error and modulus of continuity. The linear interpolation technique reduces the error by a factor of $1/\delta^{-1}$.}\label{fig:rate}
\end{figure*}

\paragraph{Step 1.~Token-wise Quantization.}
The network uses the feed-forward network to assign each input token, denoted by $\mX_{:, n}$, to a token ID, denoted by $z$, in a token-wise manner.
\begin{equation}
    \mX_{:, n} \in [0, 1]^d \to z \in \{0, 1, \dotsc, \delta^{-d} -1 \}.
\end{equation}

\paragraph{Step 2.~Contextual Mapping.} The network, given $N$ token IDs computes their sequence ID. We notice that the result of previous studies on Transformers~\citep{Yun2020Are, kim2023provable} cannot be directly applied to Looped Transformers due to the following distinctions:
\begin{itemize}[leftmargin=*]
\item \citet{Yun2020Are} employed both sparse and uniform attention mechanisms, whereas Looped Transformers are limited to a single fixed attention layer.
\item \citet{kim2023provable} used $N$ layers to store $N$ parameters required for representing the target function, whereas Looped Transformers have a fixed parameter size.
\end{itemize}
Notably, we found that Looped Transformers with $N$-loops can compute contextual mapping. Let $\vz \in \{0, 1, \dotsc, \delta^{-d}-1\}^N$ represent a sequence consist of $N$ ordered and distinct token IDs, satisfying $\vz_1 > \vz_2 > \cdots > \vz_N$. The network then maps $\vz$ to a sequence ID through an inner product with $\vu = (\delta^{-d(N-1)}, \dotsc,\delta^{-d}, 1)$, which satisfies
\begin{equation}
    |\vu^\top \vz- \vu^\top \vz'| > 1,  \quad \text{if }  \vz \ne \vz'.
\end{equation}
This guarantees that the network assigns distinct sequence IDs for different $\vz$. Combined with token IDs, the network computes contextual mapping. The key idea is that the network requires only $\delta^{-d}$ to represent $\vu$, allowing it to be implemented with Looped Transformers.

\paragraph{Step 3.~Function Value Mapping.}
The network maps the contextual token IDs into the target embeddings in a token-wise manner, using $K-1$ loops to sequentially map $k = 0, 1, \dots, K-1$ to $\tilde{\vy_k}\in \R^d$, which approximates $\vy_k$, in each iteration.
In our constructive proofs, we design both the set of contextual token IDs and their ordering.

Weight-tied feed-forward networks cannot map accurately, and the error can only be bounded by the maximum difference between adjacent contextual token embeddings, \ie
\begin{equation}\label{eq:tok_bound}
   |(\tilde{\vy}_k - \vy_{k})_i| \leq \max_{k'\in \all}|( \vy_{k'} - \vy_{k'-1})_i | 
\end{equation}
holds for $k \in \all$ and $i = 1, \dots, d$.

Generally, the following inequality holds, for $\vx \in \mathbb{R}^m$, 
\begin{equation}\label{eq:main_norm}
    \max_i |\vx_i| \leq \|\vx\|_p \leq m^{\frac{1}{p}} \max_i |\vx_i|.
\end{equation}
That is, by controlling the $p$-norm, $\| \vy_{k} - \vy_{k-1} \|_p$, the error in Eq.~\ref{eq:tok_bound} can bounded. We require $\all$ to be designed such that the differences between neighboring contextual token embeddings are bounded \wrt the $p$-norm. 

To illustrate our idea, consider the following sentences: 
\begin{enumerate}
    \item[(1)] I \underline{write} papers. \ ; \ I write \underline{papers}. \ \ (different token ID with same sequence ID)
    \item[(2)] I \underline{write} papers. \ ; \ You \underline{write} books. \ \ (same token ID with different sequence ID)
\end{enumerate}
We found that none of the moduli of continuity, defined in \cref{subsec:continuity}, alone can bound the difference between `\underline{write} and `\underline{papers}' in (1). In contrast, the error of `\underline{write}' in (2) can be bounded by the contextual continuity, $\omega^{\text{cont}}_f$.
Thus, we designed contextual token IDs such that, basically, identical or similar tokens with different sequence IDs are positioned adjacent to each other, as shown in Fig.~\ref{fig:rate}. To reduce errors in corner cases, linear interpolation is applied; further details are provided in \Cref{appendix:function}.
This allows us to obtain the following error bound.
\begin{equation}\label{eq:err_next}
    \max_{k'\in \calK} \| \vy_{k'} - \vy_{k'-1} \|_p \leq \omega^{\text{tok}}_f(\delta\sqrt{d}) + \omega^{\text{cont}}_f(\delta\sqrt{Nd}).
\end{equation}
Substituting $\vx = \vy_{k} - \vy_{k-1}$ into Eq.~\ref{eq:main_norm}, with Eq.~\ref{eq:tok_bound} and Eq.~\ref{eq:err_next}, the following result holds:
\begin{equation}\label{eq:last_0}
   |(\tilde{\vy}_k - \vy_{k})_i| \leq \omega^{\text{tok}}_f(\delta\sqrt{d}) + \omega^{\text{cont}}_f(\delta\sqrt{Nd}),
\end{equation}
for $i = 1, \dots, d$ and $k \in \all$.

\paragraph{Concatenated into a Single Transformer Layer} In the final construction, we show that the composition of the three sub-networks from Steps 1, 2, and 3 can be implemented within a single Transformer block. While our proof strategy follows \citet{zhang23}, their approach necessitates an additional layer. In contrast, we show that a single Transformer block suffices, as detailed in~\Cref{appendix:function}.

\paragraph{Deriving Approximation Rate} Lastly, we analyze the approximation error of our construction and establish the approximation rate in terms of the number of loops. 

With the triangle inequality, we obtain the following:
\begin{align}
     &\| \tilde{f} - f \|_{L^p}  \leq  \int \| \tilde{f}(\mX) - f(\mX) \|_p d\mX \\ 
     &\leq \int \| \tilde{f}(\mX) - \bar{f}(\mX) \|_p d\mX + \int \| \bar{f}(\mX) - f(\mX) \|_p d\mX \notag \\
     & \quad \quad + \mathcal{O}(N^{2/p} \delta^{d/p}) + \mathcal{O}\left(\left( (M\delta)^{-1} dN \right)^{1/p}\right)\label{eq:triangle_0}, 
\end{align}
where $\mathcal{O}(N^{2/p} \delta^{d/p})$ arises from the case where identical tokens appear in sequences, and $\mathcal{O}\left(\left( (M\delta)^{-1} dN \right)^{1/p}\right)$ results from the restriction on the norm of weights.

Considering the error within cubes in Eq.~\ref{eq:bB}, we obtain
\begin{equation}\label{eq:err_cube}
\int \| \bar{f}(\mX) - f(\mX) \|_p d\mX \leq \omega_f(\delta \sqrt{Nd}).
\end{equation}

Since, generally, the norm of sequences can be bounded by the maximum norm of the token-wise vectors as
\begin{equation}\label{eq:seq_vec}
     \| f(\mX) \|_p \leq {(Nd)}^{\frac{1}{p}}\max_{i,n} | f(\mX)_{i,n} |,
\end{equation}
the error between $\tilde{f}$ and $\tilde{f}$ can be bounded by
\begin{equation}\label{eq:last_1}
    \int \| \tilde{f}(\mX) - \bar{f}(\mX) \|_p d\mX \leq {(Nd)}^{\frac{1}{p}} \max_{k'\in \calK} | \tilde{\vy}_{k'} - \vy_{k'} |.
\end{equation}
Substituting $\vx = \tilde{\vy}_{k} - \vy_{k}$ into Eq.~\ref{eq:main_norm}, and using Eq.~\ref{eq:last_0} and Eq.~\ref{eq:last_1}, we obtain:
\begin{equation}\label{eq:triangle3}
\begin{split}
    &\int \| \tilde{f}(\mX) - \bar{f}(\mX) \|_p d\mX \\ &\leq (Nd)^{\frac{1}{p}} \Big(\omega^{\textnormal{tok}}_f(\delta\sqrt{d}) + \omega^{\textnormal{cont}}_f(\delta\sqrt{Nd}) \Big).
\end{split}
\end{equation}

We then express $\delta$ in terms of the number of loops $r$ to determine the approximation rate. We use $\delta^{-1}-1$ loops for Step 1, $N$ loops for Step 2, and $2\delta^{-(N+1)d}-1$ loops for Step 3, with $1$ loop required to connect each step \ie
\begin{equation}
r = \delta^{-1} + 2\delta^{-(N+1)d} + N.
\end{equation}
Now, $\delta$ can be bounded in terms of the number of loops as:
\begin{align}
    & \delta^{-1} + 2\delta^{-(N+1)d} = r - N \\
    &\Rightarrow \delta^{-1} \cdot 2\delta^{-(N+1)d} \ge r - N \quad (\delta^{-1} \ge 2) \\
    &\Leftrightarrow \delta \leq {\Big(\frac{r-N}{2}\Big)}^{-1/((N+1)d+1)} \label{eq:delta}.
\end{align}
By combining Eq.~\ref{eq:err_cube} and Eq~\ref{eq:triangle3} with Eq.~\ref{eq:triangle_0}, and substituting Eq.~\ref{eq:delta}, we obtain \Cref{thm:universal}. 

\section{From Theory to Practice: Introducing Timestep Encoding}\label{sec:lim} 
The theoretical result in \Cref{sec:unversalapproximators} highlights a limitation of the looped architecture. We show that a variant of architecture can overcome this limitation.

\subsection{Motivation}
\paragraph{Limitation Specific to Looped Transformers} The approximation rate in \Cref{thm:universal} includes two additional moduli of continuity, which can lead to increased errors, reflecting a limitation inherent to Looped Transformers. 

We can identify the cause of additional dependency in the
error in Eq.~\ref{eq:tok_bound}, caused by weight-tied feed-forward networks. This can be formalized as follows:
\begin{restatable}{lemma}{memorize}
\label{lemma:memorize}
Given $\vy_k \in \R^d$ for $k = 0,1, \dotsc, K-1$ with
\begin{equation*}
    |(\vy_k - \vy_{k-1})_i| \leq \varepsilon_i \quad \text{for } i=1,\dotsc,d,
\end{equation*}
there exists a feed-forward layer $\mathrm{FF}:\R^{12d}\to \R^{12d}$ with a width size of $18d$, and two affine linear maps $\calL_1:\R\to\R^{12d}$ and $\calL_2:\R^{12d}\to \R^{d}$ \st
\begin{equation}\label{eq:ineq_sec}
    |\big(\calL_2\circ (\mathrm{id} + \mathrm{FF})^{\circ(K-1)} \circ \calL_1(k) - \vy_k\big)_i| \leq \varepsilon_i,
\end{equation}
for $i=1,\dotsc,d$ and $k=0,1,\dotsc,K-1$.
\end{restatable}
This shows that large variations in the target function may lead to approximation errors, raising the question of whether inequality in Eq.~\ref{eq:ineq_sec} can be replaced with equality.

\paragraph{Improving Approximation Rate of Looped Transformers}
To eliminate this dependency, we introduce \emph{time-dependent} parameters. Specifically, we modify the feed-forward layers by adding a scaling vector for each loop step as follows:
\begin{equation*}
    \ff(\mX) \to \boldsymbol\eta(t) \odot \ff(\mX) \quad \text{for the $t$-th loops},
\end{equation*}
where $\odot$ is an element-wise product,$t \in \N$ is the loop index, and $\boldsymbol\eta(t) \in \R^d$ is the scaling parameter for each loop. This method is analogous to Hypernetworks~\citep{ha2016hypernetworks}. With the definition
\begin{equation*}
    (\mathrm{id} + \boldsymbol\eta\odot\mathrm{FF} )^{r} \coloneq (\mathrm{id} + \boldsymbol\eta(r)\odot\mathrm{FF})\circ \cdots \circ (\mathrm{id} + \boldsymbol\eta(1)\odot\mathrm{FF}),
\end{equation*}
we show that this model can memorize labels exactly.
\begin{restatable}{theorem}{timedependentmemorize}
\label{lemma:time_dependent_memorize}
Given $\vy_k \in \R^d$ for $k=0,1, \dotsc, K-1$,
there exists a feed-forward layer $\mathrm{FF}:\R^{4d}\to \R^{4d}$ with a width size of $6d$, $\boldsymbol\eta(t)\in\R^{4d}$ for $t=1, \dotsc, K-1$, and two affine linear maps $\calL_1:\R\to\R^{4d}$ and $\calL_2:\R^{4d}\to \R^{d}$ \st
\begin{equation*}
    |\big(\calL_2\circ (\mathrm{id} + \boldsymbol\eta\odot\mathrm{FF})^{\circ(K-1)} \circ \calL_1(k) - \vy_k\big)_i| = 0,
\end{equation*}
for $i=1,\dotsc,d$ and $k=0,1,\dotsc,K-1$.
\end{restatable}
The proof is provided in \Cref{appendix:time_dependent}. For implementation, adding parameters per loop increases the total parameter count proportionally. Thus, we introduce timestep encoding.

\subsection{Timestep-Modulated Looped Transformer}
We employ timestep encodings to condition scaling parameters on the loop index (timestep). This method is inspired by adaptive instance normalization~\cite{peebles2023scalable}.

To condition on timesteps, frequency embeddings are processed through a two-layer MLP with hidden size matching the Transformer block and SiLU activation. Let $\operatorname{TE}(t)\in \R^d$ denote timestep embeddings, defined as:
\begin{equation*}
   \operatorname{TE}(t) = \mW_3 \cdot \mathrm{SiLU} (\mW_4 \cdot \operatorname{PE}(t) + \vb_4 ) + \vb_3,
\end{equation*}
where $\mW_3, \mW_4 \in \R^{d \times d}, \vb_3, \vb_4 \in \R^d$, and $\operatorname{PE}(t) \in \R^d$ denotes the timestep encoding function that maps the timestep into a $d$-dimensional embedding, \st
\begin{align*}
    \operatorname{PE}(t)_{2i} &= \sin(t / 10000^{2i/d}),\\ 
    \operatorname{PE}(t)_{2i+1} &= \cos(t / 10000^{2i/d}).
\end{align*}
We use the root mean square layer normalization (RMSNorm)~\cite{zhang2019root}, which is widely used in several recent LLMs~\cite{llama,gemma}, defined as:
\begin{equation*}
    \begin{split}
    \boldsymbol\bar{\vx} = \boldsymbol\alpha \odot \frac{\vx}{\text{RMS}(\vx)} , \ \text{where} \ \text{RMS}(\vx) = \sqrt{\frac{1}{d} \sum_{i=1}^{d} \vx_i^2},
    \end{split}
\end{equation*}
where $\boldsymbol\alpha\in\R^d$ is a gain parameter for rescaling. We define time-dependent RMSNorm, denoted by $\LN$, as:
\begin{equation*}
    \LN(\vx, t) = \boldsymbol\alpha(t) \odot \frac{\vx}{\text{RMS}(\mathbf{\vx})} 
\end{equation*}
where $\boldsymbol\alpha(t) \in \R^d$ is a time-dependent parameter generated by a network. With scaling parameters, the time-dependent Transformer block is defined as follows:
\begin{align*}
    \mX' &= \mX + \boldsymbol\gamma_1(t) \odot \attn(\bm{\LN}_1(\mX, t)),\\
    \tf(\mX, t) &= \mX' + \boldsymbol\gamma_2(t) \odot \bm{\ff}(\bm{\LN}_2(\mX', t)),
\end{align*}
where $\gamma_1(t), \gamma_2(t) \in \R^d$ are time-dependent parameters applied token-wise, as well as RMSNorm.

The time-dependent vector parameters are generated as:
\begin{equation*}
    \boldsymbol\alpha_1(t), \boldsymbol\alpha_2(t), \boldsymbol\gamma_1(t), \boldsymbol\gamma_2(t) = \mW_5 \cdot \SiLU(\operatorname{TE}(t)) + \vb_5,\\
\end{equation*}
where $\mW_5 \in \R^{4d \times d}$ and $\vb_5 \in \R^d$.

\section{Experiments}
\begin{table*}[ht!]
    \centering
    \renewcommand{\arraystretch}{0.7} 
    \setlength{\tabcolsep}{12pt} 
    \caption{Test accuracy for reasoning tasks. Performance improves as the number of loops increases..}
    \vskip 0.1in
    \begin{tabular}{@{}l|c|cccc|cccc@{}}
        \toprule
        \multirow[c]{2}{*}{\textbf{Task}} & \textbf{TF} & \multicolumn{4}{c|}{\textbf{Looped TF}} & \multicolumn{4}{c}{\textbf{w/ Timestep Encoding}} \\
        \cmidrule(lr){2-2} \cmidrule(lr){3-6} \cmidrule(lr){7-10}
        & \textbf{L=6} & \textbf{r=4} & \textbf{r=8} & \textbf{r=16} & \textbf{r=32} & \textbf{r=4} & \textbf{r=8} & \textbf{r=16} & \textbf{r=32} \\
        \midrule
        Sudoku & 0.0 & 0.0 & 0.0 & 65.6 & 87.9 & 0.0 & 0.0 & 62.0 & \textbf{90.2} \\
        Countdown & 53.8 & 28.3 & 52.7 & 81.0  & 88.1 & 33.2 & 54.4 & 80.2 & \textbf{90.5} \\
        \midrule
        & \textbf{L=12} & \textbf{r=5} & \textbf{r=10} & \textbf{r=50} & \textbf{r=100} & \textbf{r=5} & \textbf{r=10} & \textbf{r=50} & \textbf{r=100} \\
        \midrule
        LCS (60) & 70.0 & 66.0 & 81.8 & 98.6 & 96.9 & 68.5 & 80.5 & \textbf{99.3} & 97.1 \\
        LCS (100) & 39.8 & 39.6 & 45.1 & 93.5 & 98.2 & 36.7 & 45.6 & 98.1 & \textbf{98.6} \\
        \midrule
        ED (40) & 54.2 & 41.4 & 57.9 & 85.4 & 90.4 & 44.8 & 63.5 & 94.5 & \textbf{96.1} \\
        ED (60) & 41.4 & 23.8 & 32.6 & 47.3 & 47.7 & 26.6 & 38.9 & 57.3 & \textbf{88.3} \\
        \bottomrule
    \end{tabular}
    \label{tab:dp_acc}
\end{table*}
This section presents experimental results supporting our theoretical findings. We used Looped Transformers with varying numbers of loops, both with and without timestep encoding, and compared to standard Transformers.
We assess approximation capabilities based on test evaluation, as we observe a strong correlation between train and test performance. The details are provided in \Cref{app:details_of_exp}.
%

\subsection{Problem Setting}
We evaluate the model on two types of tasks. The first consists of reasoning problems known to be challenging for standard Transformers. These are used to examine whether increasing the number of loops and incorporating timestep encodings can enhance performance.
The second includes core Transformer benchmarks, such as in-context learning and language modeling. 
 
\subsubsection{Reasoning Tasks}
\textbf{Dynamic Programming } is a method for solving complex problems by breaking them down into simpler sub-problems. We use edit distance (ED) and longest common subsequence (LCS) tasks with varying input lengths. Each task has $10^6$ train samples and $10^3$ test samples.

\textbf{Sudoku} is a constrained satisfaction problem that involves filling a \(9 \times 9\) grid with digits from $1$ to $9$, such that each digit appears exactly once in every row, column, and predefined \(3 \times 3\) sub-grid. The grid is flattened into a sequence representation. Unlike \cite{yang2023learning}, we use the dataset from \cite{radcliffe2020sudoku}, sampling $3$M instances for training and $100$K for testing.

\textbf{Countdown} is a game in which a given set of input numbers must be combined using basic arithmetic operations to reach a target number~\cite{yao2023tree,gandhi2024stream}. We consider cases with $4$ input numbers and target numbers ranging from $10$ to $100$, where $10\%$ of the target numbers are reserved for evaluation. We generate $5$M samples for training and $1$K samples for testing.

\subsubsection{In-Context and Language Modeling}
The in-context learning problem is to learn the function class from a given sequence, which was investigated with Looped Transformers~\cite{yang2024looped} without timestep encodings. We use decision tree functions. For the language modeling task, we use the WikiText-103~\citep{merity2017pointer} dataset, containing over $100$ million tokens from Wikipedia articles. Details are in~\Cref{app:incontext} and \Cref{app:langugae}.

\subsection{Results}
The results in Table~\ref{tab:dp_acc} demonstrate that increasing the number of loops improves performance on reasoning tasks, with higher loop counts significantly outperforming standard Transformers. Furthermore, incorporating timestep encodings leads to additional gains; in particular, for the edit distance task with input size $n=60$, the model with loop counts $r=100$ achieves significantly better performance when timestep encodings are incorporated. 

\begin{table}[ht!]
\renewcommand{\arraystretch}{0.7}
\centering
\caption{MSE (↓) on the in-context learning task.}
\vskip 0.1in
\label{tab:incontext}
\begin{tabular}{l|ccc}
\toprule
& \textbf{TF L=12} & \textbf{Looped r=12} & \textbf{w/ Timestep r=12} \\
\midrule
Test & 8.6e-03 & 1.4e-02 & \textbf{1.7e-03} \\
\bottomrule
\end{tabular}
\end{table}
\begin{table}[ht!]
\renewcommand{\arraystretch}{0.7}
\centering
\caption{Perplexity (↓) on the WikiText-103 dataset.}
\vskip 0.1in
\label{tab:perplexity}
\begin{tabular}{l|ccc}
\toprule
& \textbf{TF L=12} & \textbf{Looped r=24} & \textbf{w/ Timestep r=24} \\
\midrule
Train & \textbf{15.9} & 17.1 & \textbf{15.9} \\
Test & 20.5 & 20.6 & \textbf{19.6} \\
\bottomrule
\end{tabular}
\end{table}
As evidenced by the results in \Cref{tab:incontext} and \Cref{tab:perplexity}, the use of timestep encodings leads to performance gains in both in-context learning and language modeling.

\section{Conclusion} We establish the approximation rate of Looped Transformers with respect to the number of loops and the moduli of continuity of the target function. Our analysis reveals a limitation of Looped Transformers, which is addressed by timestep encodings. To the best of our knowledge, this study is the first to investigate the function approximation capabilities of Looped Transformers.
Extending the analysis to multiple layers, varying input lengths, and characterizing optimal memorization capacity presents promising avenues for future research. Beyond expressivity, investigating estimation performance and enhancing training stability constitute important challenges moving forward.

\section*{Acknowledgements}
This work was supported by JSPS KAKENHI Grant Number JP24H00709.

\section*{Impact Statement}
This paper presents work whose goal is to advance the field of 
Machine Learning. There are many potential societal consequences 
of our work, none which we feel must be specifically highlighted here.

\bibliography{icml2025_paper}
\bibliographystyle{icml2025}

\newpage
\appendix
\onecolumn

\section{Proofs for Theorem \ref{thm:universal}}\label{appendix:function}

The main theorem incorporates a restriction on the norm of weights, leading to errors when approximating discontinuous functions, such as step functions with ReLU or hardmax functions with softmax.
We first establish the approximation rate assuming that weights can take arbitrary precision real values, as outlined below. Then, we account for the bit complexity of bounded weights to complete the proof of the main \Cref{thm:universal}.
\begin{theorem}\label{thm:rate}
Given a function $f \in \mathcal{F}_{\mathrm{PE}}([0,1]^{d \times N})$, $r > N$, there exists a Looped Transformer, composed of $\mathrm{TF}:\R^{(17d+9)\times N}\to\R^{(17d+9)\times N}$ with two heads, head size $1$, a width size of $q=49d+25$, and two affine linear maps $\mathcal{L}_1:\R^d\to\R^{17d+9}$ and $\mathcal{L}_2:\R^{17d+9}\to\R^d$ \st
\begin{equation}
    \big\|\bm{\mathcal{L}}_2\circ \mathrm{TF}^{\circ r}\circ \bm{\mathcal{L}}_1-f\big\|_{L^p}
    \leq (Nd)^{\frac{1}{p}} \Big(\omega^{\textnormal{tok}}_f(\delta\sqrt{d}) + \omega^{\textnormal{cont}}_f(\delta\sqrt{Nd})\Big) + \omega_f(\delta \sqrt{Nd}) + \mathcal{O}(N^{2/p} \delta^{d/p}),
\end{equation}
where $\delta = \big((r-N)/2\big)^{-1/((N+1)d+1)}$.
\end{theorem}

\subsection{Proof of \Cref{thm:rate}}
\begin{proof}
Since any continuous function can be approximated by a piecewise constant function with arbitrarily small errors, we approximate $f \in \mathcal{F_{\mathrm{PE}}}([0,1]^{d \times N})$ with piece-wise constant function $\bar{f}: [0,1]^{d \times N} \to \R^{d \times N}$.
We choose $\delta^{-1} \in \mathbb{N}, \delta^{-1} \ge 2$, determining how finely the input is divided: we divide the input space $[0,1]^{d \times N}$ into \textit{$\delta$-discretized cubes}, denoted by $\{\bm{Q}_{\bB}\}$ for ${\bB \in \{0,1,\dotsc,\delta^{-1} - 1\}^{d\times N}}$ defined by
\begin{equation}
\bm{Q}_{\bB}\coloneqq\Big\{\rmX \in [0,1]^{d \times N}:\rmX_{i,n}\in\big[\bB_{i,n}\delta, (\bB_{i,n}+1)\delta \big), \quad i=1,2,\dotsc,d, \ n=1,2,\dotsc,N \Big\}.
\end{equation}

Each cube \(\bm{Q}_{\bB}\) is associated with a representative point \(\hat{\mX}_{\bB}\), defined as the vertex of \(\bm{Q}_{\bB}\) with the minimal \(\ell_1\) norm. Then, we define the piecewise constant function $\bar{f}$ for $\mX \in [0,1]^{d \times N}$ as
\begin{equation}
    \bar{f}(\mX) \coloneq f(\hat{\mX}_{\bB}). 
\end{equation}
Since we can bound the error within each cube, we have:
\begin{equation}\label{eq:cube-err}
\max_{\mX \in [0,1]^{d \times N}} \{\| \bar{f}(\mX) - f(\mX) \|_p\} \leq \omega_f(\delta\sqrt{Nd}).
\end{equation}

Our construction consists of three steps to approximate \(\bar{f}\), as outlined below. 
\begin{enumerate}[leftmargin=1.1em]
    \item The network, with $\delta^{-1}-1$ loops, maps the input space $[0, 1]^d$ token-wise to the coordinates $\bm{\beta} \in \{0, 1, \dots, \delta^{-1}-1\}^d$ of discretized cubes, and then bijectively maps these coordinates to integers, representing \textit{token IDs} in the set $\{0, 1, \dots, \delta^{-d}-1\}$, using a $\delta^{-1}$-base system; for example, if \(d = 2\) and \(\delta^{-1} = 3\), then coordinates \((\beta_1, \beta_2) = (2, 1)\) are mapped to the integer \(2 \times 3^1 + 1 \times 3^0 = 7\).
    \item The network, with $N$ loops, computes a contextual mapping from the set of $N$ distinct token IDs into the set of \emph{contextual token ID}. \emph{Contextual token IDs} refer to token IDs assigned to each token within the context of a \textit{sequence ID}.
    \item The network, with $2\delta^{-(N+1)d}-1$ loops, approximately maps \emph{contextual token IDs} into the output embeddings of each token in a token-wise manner. To achieve a small approximation error, the network has to be designed so that neighboring IDs correspond to similar output token embeddings. Furthermore, \emph{dummy indices} are used to reduce the error.
\end{enumerate}
The details for each step are provided below.

\paragraph{Step 1.~Token-wise Quantization.}
The input space for each token $\vx \in [0,1]^{d}$ are divided into $\delta$-discretized cubes denoted by $\{Q_{\bm{\beta}}\}$ for ${\bm{\beta} \in \{0,1,\dotsc,\delta^{-1} - 1\}^d}$, defined as
\begin{equation}
Q_{\bm{\beta}}\coloneqq\Big\{\vx \in [0,1]^d:\vx_i\in\big[\bm{\beta}_i\delta,\ (\bm{\beta}_i + 1)\delta \big), \quad \text{for all } i=1,2,\dotsc,d\Big\}.
\end{equation}
By \Cref{lemma:quantize}, 
there exists a feed-forward layer $\ff^{(1)} : \mathbb{R}^{5d} \to \mathbb{R}^{5d}$ of width size $q = 7d$, and two affine linear maps $\mathcal{L}^{(1)}_1 : \mathbb{R}^d \to \mathbb{R}^{5d}$ and $\mathcal{L}'^{(1)}_2 : \mathbb{R}^{5d} \to \mathbb{R}^d$ such that 
\begin{equation}
    \mathcal{L}'^{(1)}_2\circ \big(\mathrm{id} + \ff^{(1)}\big)^{\circ (\delta^{-1} - 1)}\circ\mathcal{L}^{(1)}_1(\vx) = \bm{\beta} \quad \st \quad \vx\in Q_{\bm{\beta}}.
\end{equation}

In addition, we need to bijectively map the $d$-dimensional vector $\bm{\beta}$ to an integer \emph{token ID}, denoted by $z$. We use a \textit{$\delta^{-1}$-base system}: we define the vector $\vu_{(\delta^{-1})} \in \R^d$ as
\begin{equation}
\vu_{(\delta^{-1})} \coloneqq (\delta^{-(d-1)}, \delta^{-(d-2)}, \dotsc, \delta^{-1}, 1)^\top,
\end{equation}
and define \( z \) as
\begin{equation}\label{eq:tok_id}
z \coloneq \vu_{(\delta^{-1})}^\top \bm{\beta} \quad \in \{0, 1, \dotsc, \delta^{-d} - 1\}.
\end{equation}
To implement this, we define an affine linear map $\mathcal{L}^{(1)}_2 : \mathbb{R}^{5d} \to \mathbb{R}$ via 
\begin{equation}
    \mathcal{L}^{(1)}_2(\vx) = \vu_{(\delta^{-1})}^\top \mathcal{L}'^{(1)}_2(\vx).
\end{equation}
Thus, we have
\begin{equation}
\left(\mathcal{L}^{(1)}_2\circ(\ff^{(1)}_1)^{\circ {(\delta^{-1}-1)}} \circ \mathcal{L}^{(1)}_1(\vx)\right)_n = \vu_{(\delta^{-1})}^\top \bm{\beta} = z  \quad \st \quad \vx\in Q_{\bm{\beta}}.
\end{equation}

We establish an upper bound on the maximum distance in input space between adjacent token IDs to derive the approximation error for the following steps. Define the input cubes corresponding to each token ID $z$ as follows:
\begin{equation}
Q_{z}\coloneqq\Big\{\vx \in [0,1]^d:\vx_i\in\big[\bm{\beta}_i\delta,\ (\bm{\beta}_i + 1)\delta\big) \quad \text{for } i=1,2,\dotsc,d \quad \st \quad z = \vu_{(\delta^{-1})}^\top \bm{\beta} \Big\}.
\end{equation}
Then we have
\begin{equation}\label{eq:seq0}
    \max_{z, \vx \in Q_{z}, \vx' \in Q_{z-1}}\|\vx - \vx'\|_2 \leq 
    \begin{cases}
      \delta\sqrt{d}, &\quad \text{if } \bm{\beta}_d \in \{1, 2, \dotsc, \delta^{-1} - 1\}, \\
      \sqrt{d}, &\quad \text{if } \bm{\beta}_d = 0.
    \end{cases}
\end{equation}
Informally, in this \(\delta^{-1}\)-based representation, the least significant digit corresponds to the index of the \(d\)-th dimension, \(\bm{\beta}_d\). As the token ID increments sequentially, the index in the \(d\)-th dimension increases as \(0, 1, 2, \dots, \delta^{-1}-1\), while the higher-order digits remain unchanged. Consequently, consecutive token IDs correspond to tokens that are ``similar'' in the \(d\)-dimensional space, with a maximum distance of \(\delta\sqrt{d}\). However, when a carry occurs, the higher-order digits may change significantly, leading to cases where tokens that are not adjacent in input space become adjacent in their indices. In such cases, the distance is only bounded by \(\sqrt{d}\).

\paragraph{Step 2.~Contextual Mapping.}
The networks, with $N$-loops, map the list of $N$ token IDs, denoted by $\vz \in \{0, 1, \dotsc, \delta^{-d} - 1\}^N$, into sequence IDs bijectively. Combined with token IDs, the network computes contextual mapping.

We consider only the case where all \(N\) input tokens are distinct, disregarding other cases, as they can be treated as negligible when \(\delta\) is small. The number of subsets in which at least one of the \(N\) tokens is duplicated is given by
\begin{equation}
    (\delta^{-d})^N - \big(\delta^{-d} \cdot (\delta^{-d}-1) \dotsm (\delta^{-d}-N+1)\big) < \frac{N(N-1)}{2}\delta^{-(N-1)d},
\end{equation}
when \(\delta\) is sufficiently small.
The volume of these subsets is bounded by
\begin{equation}
    \frac{C\delta^{-(N-1)d}}{\delta^{-Nd}} = \mathcal{O}(N^2\delta^{d}).
\end{equation}
Thus, the error with respect to the \(L^p\) norm is bounded by \(\mathcal{O}(N^{2/p} \delta^{d/p})\).

Let $\sL_{\delta}$ denote the set of $N$ distinct token IDs, \ie
\begin{equation}
\sL_{\delta} \coloneq \{\vz \in \{0, 1, \dotsc, \delta^{-d} - 1\}^N \mid \vz_{i} \neq \vz_{j} \text{ for all } i \neq j\}.
\end{equation}
Due to permutation equivariance, we can assume without loss of generality that elements of $\vz \in \sL_{\delta}$ are ordered, \ie, $\vz_1 > \vz_2 > \cdots > \vz_N$.
Define $\vu_{(\delta^{-d})} \coloneqq (\delta^{-(N-1)d}, \dotsc, \delta^{-d}, 1)^\top$, which satisfy
\begin{equation}
    |\vu_{(\delta^{-d})}^\top \vz- \vu_{(\delta^{-d})}^\top \vz'| > 1, \quad \text{for any } \vz, \vz' \in \sL_{\delta} \text{ with } \vz \neq \vz'.
\end{equation}
This mapping, $\vu_{(\delta^{-d})}^\top \vz$, represents $\vz$ in a $\delta^{-d}$-base system.
Then, we define sequence ID for $\vz \in \sL_{\delta}$ as:
\begin{equation}\label{eq:seqenceID}
\mathrm{s}(\vz) \coloneqq \vu_{(\delta^{-d})}^\top \vz = \sum_{n=1}^{N} \vz_n \delta^{-(N-n)d}.
\end{equation}
By \Cref{lemma:contextmap}, there exists a Transformer block \(\mathrm{TF'^{(2)}}: \mathbb{R}^{5 \times N} \to \mathbb{R}^{5 \times N}\) with single head, head size $s=1$, and width size $q=3$, and two affine linear maps \(\mathcal{L}'^{(2)}_1: \mathbb{R} \to \mathbb{R}^5\) and \(\mathcal{L}'^{(2)}_2: \mathbb{R}^5 \to \mathbb{R}\) such that
\begin{equation}
\bm{\mathcal{L}}'^{(2)}_2 \circ {\big(\tf'^{(2)}\big)}^{\circ N} \circ \bm{\mathcal{L}}'^{(2)}_1(\vz^\top) = \mathrm{s}(\vz) \cdot \mathbf{1}_N^\top.
\end{equation}

Furthermore, we have to add \emph{dummy indices} to alleviate the approximation error caused by the looped architecture in Step 3.
Recall that ${\bB \in \{0,1,\dotsc,\delta^{-1} - 1}\}^{d\times N}$ represents the coordinates of the inputs. Let $\mZ_{\bB} \in \{0,1,\dotsc,\delta^{-1} - 1\}^{d \times N}$ denote the ordered coordinates of $\bB$ where the tokens are ordered by their token IDs, i.e., ${(\vu_{(\delta^{-1})}^\top \mZ_{\bB})}_1 > {(\vu_{(\delta^{-1})}^\top \mZ_{\bB})}_2 > \cdots > {(\vu_{(\delta^{-1})}^\top \mZ_{\bB})}_N$, in other words, $\vz = \vu_{(\delta^{-1})}^\top \mZ_{\bB}$. Recall that we consider only the case where all $N$ input tokens are distinct. By redefining the sequence ID of Eq.~\ref{eq:seqenceID} for $\bB$ instead of $\vz$, sequence IDs in $\delta^{-d}$-base can be rewritten in the $\delta^{-1}$-base system as follows:
\begin{align}
\mathrm{s}(\bB) &\coloneq \vu_{(\delta^{-d})}^\top (\vu_{(\delta^{-1})}^\top \mZ_{\bB}) \\
&= \sum_{i=1}^{d}\sum_{n=1}^{N} {(\mZ_{\bB})}_{i,n} \delta^{-\big((N-n)d+(d-i)\big)}.
\end{align}
Then, we define \textit{extended sequence IDs} as:
\begin{align}
\mathrm{s}_{\text{valid}}(\bB) &\coloneq 2\mathrm{s}(\bB) - {(\mZ_{\bB})}_{d,N} 
\end{align}
and the \emph{dummy sequence IDs} as:
\begin{equation}\label{eq:ex_seq_id}
\mathrm{s}^{b}_{\text{dummy}}(\bB) \coloneq \mathrm{s}_{\text{valid}}(\bB) + b.
\end{equation}
for $b \in \{ \delta^{-1}, \delta^{-1}+1, \dotsc, 2\delta^{-1}-1\}$. Then, define each set as follows:
\begin{align*}
    \calS_{\text{valid}} &\coloneqq \left\{ \mathrm{s}_{\text{valid}}(\bB) : \bB \in \{0,1,\dotsc,\delta^{-1} - 1\}^{d \times N} \right\}, \\
    \calS_{\text{dummy}} &\coloneqq \left\{ \mathrm{s}^{b}_{\text{dummy}}(\bB) : \bB \in \{0,1,\dotsc,\delta^{-1} - 1\}^{d \times N},\ b \in \{ \delta^{-1}, \delta^{-1}+1, \dotsc, 2\delta^{-1}-1 \} \right\}.
\end{align*}
Recalling that ${(\mZ_{\bB})}_{d,N} \in \{0, \ldots, \delta^{-1}-1 \}$, we observe that
\begin{align}
    \calS_{\text{valid}} \cap \calS_{\text{dummy}} &= \emptyset, \label{eq:seq_set_1} \\ \calS_{\text{valid}} \cup \calS_{\text{dummy}} &= \{0, 1, \dotsc, 2\delta^{-Nd}-1\}. \label{eq:seq_set_2}
\end{align}
We define the input cubes for each valid sequence ID $s \in \calS_{\text{valid}}$ as follows:
\begin{equation}
\bm{Q}_{s}\coloneqq\Big\{\rmX \in [0,1]^{d \times N}:\mX \in \bm{Q}_{\bB}, \quad \st \quad s = \mathrm{s}'(\bB) \Big\}.
\end{equation}
Analogous to Eq.~\ref{eq:seq0}, we have
\begin{equation}\label{eq:sbound}
    \max_{s \in \calS_{\text{valid}}, \mX \in \bm{Q}_{s}, \mX' \in \bm{Q}_{s-1}}\|\mX - \mX'\|_2 \leq 
    \begin{cases}
    \delta\sqrt{Nd} & \quad \text{if } \bB_{d,N} \in  \{1, 2, \dotsc, \delta^{-1} - 1\}, \\
    \sqrt{Nd} &\quad \text{if } \bB_{d,N} = 0.
    \end{cases}
\end{equation}

To implement this, we slightly modified $\mathrm{TF'^{(2)}}$. By \Cref{corollary:contextmap}, there exists a Transformer block \(\mathrm{TF}^{(2)}: \mathbb{R}^{8 \times N} \to \mathbb{R}^{8 \times N}\) with two heads, head size $s=1$, and width size $q=5$, and two affine linear maps \(\mathcal{L}^{(2)}_1: \mathbb{R}^2 \to \mathbb{R}^8\) and \(\mathcal{L}^{(2)}_2: \mathbb{R}^8 \to \mathbb{R}\) \st
\begin{equation}
\bm{\mathcal{L}}^{(2)}_2 \circ {\big(\tf^{(2)}\big)}^{\circ N} \circ \bm{\mathcal{L}}^{(2)}_1\left( \left[
    \begin{array}{ccc}
        {\vz}^\top  \\
        {\mZ_{d,:}}  \\
    \end{array}\right] \right) = (2\mathrm{s}(\vz) - \mZ_{d,N}) \cdot \mathbf{1}_N^\top.
\end{equation}

\paragraph{Step 3.~Token-wise Function Value Mapping.}
In Steps 1 and 2, the network receives a token ID and extended sequence ID as input for each token, collectively forming a \emph{contextual token ID}. With 2$\delta^{-(N+1)d}-1$ loops, the network approximately maps \emph{contextual token IDs} to the output token embeddings of the target function.

To construct contextual token IDs, we define a bijective affine linear mapping $\mathcal{L}^{(3)}_0: \N^{2} \to \N$ as follows:
\begin{equation}\label{eq:context_id}
    \mathcal{L}^{(3)}_0(z, s) \coloneq 2z\delta^{-Nd} + s,
\end{equation}
where $z$ represents a token ID, defined in Eq.~\ref{eq:tok_id}, and $s$ represents an sequence ID. Recall that $z \in \{0, 1, \dotsc, \delta^{-d} - 1\}$ and sequence IDs are less than $2\delta^{-Nd}$, so informally, it's as if we are adding another digit, $z$, as the most significant digit in a $\delta^{-d}$-based system.
Define the set of contextual token IDs as:
\begin{equation}
	\calK_{\text{valid}} \coloneqq \bigg\{\mathcal{L}^{(3)}_0\big(z, s) : z \in \{0, 1, 2, \dotsc, \delta^{-d}-1\}, s \in \calS_{\text{valid}}\ \bigg\}.
\end{equation}
and dummy contextual token ID as
\begin{equation}
	\calK_\text{dummy} \coloneqq \bigg\{ \mathcal{L}^{(3)}_0\big(z, s) : z \in \{0, 1, 2, \dotsc, \delta^{-d}-1\}, s \in \calS_{\text{dummy}}\ \bigg\}.
\end{equation}

From Eq.~\ref{eq:seq_set_1} and Eq.~\ref{eq:seq_set_2}, the following holds:
\begin{align}
    \calK_{\text{valid}} \cap \calK_{\text{dummy}} &= \emptyset, \\ 
    \calK \coloneq \calK_{\text{valid}} \cup \calK_{\text{dummy}} &= \{0, 1, \dotsc, 2\delta^{-(N+1)d}-1\}.
\end{align}

We now define the target output embedding for each ID. Let $\vy_k\in \R^d$ denote the contextual token embedding corresponding to each contextual token ID, defined as follows:
\begin{equation}
    \vy_k \coloneq 
    \begin{cases} 
    \bar{f}(\hat{\mX}_{\bB})_{:, n} \quad \st \quad \mathcal{L}^{(3)}_0\big(\vz_n, \mathrm{s}'(\bB)\big) = k & \text{for } k \in \calK_{\text{valid}}, \\
    \operatorname{lin\_interp}\left(\vy_{\operatorname{near^-}(k,\calK_{\text{valid}})}, \vy_{\operatorname{near^+}(k, \calK_{\text{valid}})}, k-\operatorname{near^-}(k, \calK_{\text{valid}}), \delta^{-1}\right) & \text{for } k \in \calK_{\text{dummy}},
    \end{cases}
\end{equation}
where the nearest functions are defined as
\begin{equation}
    \operatorname{near^+}(a, \mathcal{S})  \coloneqq \argmin_{b \in \mathcal{S}, b > a} \left| a - b \right|, \quad
    \operatorname{near^-}(a, \mathcal{S})  \coloneqq \argmin_{b \in \mathcal{S}, b < a} \left| a - b \right|,
\end{equation}
and a function $\operatorname{lin\_interp}$ is defined by
\begin{equation}
\operatorname{lin\_interp}(a, b, t, n) \coloneqq a + \frac{t}{n} (b - a).
\end{equation}
The illustration of $\vy$ is shown in Fig.~\ref{fig:rate}.

Thanks to our design of $\calK$, the error between neighboring contextual token embeddings can be bounded as follows.
There are two types of error: the variation induced by contextual perturbation and the variation induced by token perturbation, or both.
The examples of each pattern are shown in Fig.~\ref{fig:rate}, such that
\begin{enumerate}
    \item[(1)] I \underline{draft} papers. \ ; \ I \underline{write} papers. \ \ (perturbation of token)
    \item[(2)] He \underline{writes} papers. \ ; \ Mozart \underline{writes} music. \ \ (perturbation of context)
    \item[(3)] He \underline{drinks} coffee. ; He \underline{drinks} coffee. \ \ (perturbation of both token and context)
\end{enumerate}

Recall that there are two types of adjacency that generally have small errors, with a few instances causing large errors when `carryover' occurs, as stated in Eq.~\ref{eq:seq0} and Eq.~\ref{eq:sbound}. The key point of our design of $\calK$ is that when a large variation occurs, linear interpolation inevitably takes place to smooth out the steep changes between adjacent indices.

Thus for token ID in Eq.~\ref{eq:seq0}, if a small variation of token ID, with same context, in input space, $\delta\sqrt{d}$, occurs, the error $e_k = \| \vy_{k} - \vy_{k-1} \|_p$ in the output contextual token embedding can be bounded by the modulus of token continuity as
\begin{equation}
    e_k \le \omega^{\text{tok}}_f(\delta\sqrt{d}).
\end{equation}
In contrast, if a large variation in token input space, $\sqrt{d}$, occurs in the token input space, the error can be bounded using linear interpolation with $\delta^{-1}$ intermediate points as:
\begin{equation}
    e_k \le \delta \omega^{\text{tok}}_f(\sqrt{d}).
\end{equation}

The same holds for sequence IDs in Eq.~\ref{eq:sbound}. That is, since the variation in context is bounded by sequence variation, the difference in adjacent contextual token IDs caused by perturbations in context is bounded as
\begin{equation}
    e_k \le \omega^{\text{cont}}_f(\delta\sqrt{Nd}), \quad \text{or } \quad e_k \le \delta \omega^{\text{cont}}_f(\sqrt{Nd}).
\end{equation}
for $k=0, 1, \dots, K-1$.

Since that 
\begin{equation}
    \omega_f^{\text{cont, tok}}(n \cdot t) \leq n \cdot \omega_f^{\text{cont, tok}}(t)
\end{equation}
for any $n \in \N$ and $t \in [0, \infty)$, with $\delta<1$, it follows that (note that it holds with opposite inequality due to $\delta<1$)
\begin{equation}
    \delta\omega^{\text{cont}}_f(\sqrt{Nd}) \le \omega^{\text{cont}}_f(\delta\sqrt{Nd}) \quad \text{ and } \quad \delta\omega^{\text{cont}}_f(\sqrt{Nd}) \le \omega^{\text{cont}}_f(\delta\sqrt{Nd}).
\end{equation}

Considering the maximum difference when both token and context perturbations occur, we have, with the triangle inequality, 
\begin{equation}\label{eq:max_k}
    \max_{k'\in \calK} \| \vy_{k'} - \vy_{k'-1} \|_p \leq 
    \delta \left(\omega^{\text{tok}}_f(\sqrt{d}) + \omega^{\text{cont}}_f(\sqrt{Nd}) \right)
    \leq \omega^{\text{tok}}_f(\delta\sqrt{d}) + \omega^{\text{cont}}_f(\delta\sqrt{Nd}).
\end{equation}

Generally, the following inequality holds, for any vector $\vx \in \mathbb{R}^d$, 
\begin{equation}\label{ineq:norm}
    \max_i |\vx_i| \leq \|\vx\|_p \leq d^{\frac{1}{p}} \max_i |\vx_i|.
\end{equation}
Substituting $\vx = \vy_{k} - \vy_{k-1}$ into the above inequality results in
\begin{equation}\label{eq:max_i_k}
    \max_i |(\vy_{k} - \vy_{k-1})_i| \leq \| \vy_{k} - \vy_{k-1} \|_p.
\end{equation}

By \Cref{lemma:memorize}, there exists a feed-forward layer $\mathrm{FF}^{(3)}:\R^{12d}\to \R^{12d}$ of width size $18d$ and two affine linear maps $\calL^{(3)}_1:\R\to\R^{12d}$ and $\calL^{(3)}_2:\R^{12d}\to \R^{d}$ such that
\begin{equation}
    |\Big(\calL^{(3)}_2\circ{(\mathrm{id} + \ff^{(3)})}^{(K-1)} \circ \calL^{(3)}_1(k) - \vy_{k}\Big)_i|  \leq \max_{k'\in \calK} |(\vy_{k'} - \vy_{k'-1})_i|,
\end{equation}
for $i = 1, 2, \dotsc, d$ and $k=0,1,\dotsc,K-1$, where $K=2\delta^{-(N+1)d}$. 

Let $\tilde{\vy}_k \in \R^d$ be defined as $\tilde{\vy}_k \coloneq \calL^{(3)}_2\circ{(\mathrm{id} + \ff^{(3)})}^{(K-1)} \circ \calL^{(3)}(k)$. Then we have  
\begin{align}
    |{(\tilde{\vy}_k - \vy_{k})}_i| & \leq \max_{k'\in \calK} |(\vy_{k'} - \vy_{k'-1})_i| \\ & \leq \max_{k'\in \calK} \| \vy_{k'} - \vy_{k'-1} \|_p \quad \text{because of Eq.~\ref{eq:max_i_k}} \\
    & \leq \omega^{\text{tok}}_f(\delta\sqrt{d}) + \omega^{\text{cont}}_f(\delta\sqrt{Nd}) \quad \text{because of Eq.~\ref{eq:max_k}} \label{ineq:i_error},
\end{align}
for $i = 1, 2, \dotsc, d$ and $k=0,1,\dotsc,K-1$.

\paragraph{Concatenated into a Single Transformer Layer }
Define the input space for each step as:
\begin{equation}
\mX^{(0)} \in \R^{1 \times N}, \quad
\mX^{(1)} \in \R^{5d \times N}, \quad
\mX^{(2)} \in \R^{8 \times N}, \quad
\mX^{(3)} \in \R^{12d \times N},
\end{equation}
where $\mX^{(0)}$ act as counter.
Define $\attn: \R^{(17d+9) \times N} \to \R^{(17d+9) \times N}$ with two heads, head size $s = 1$, and width size $q = 5$ via
\begin{equation}
    \attn\left(
\left[
\begin{array}{c}
\mX^{(0)} \\
\mX^{(1)}  \\ 
\mX^{(2)}  \\ 
\mX^{(3)}  
\end{array}
\right]
\right)
= 
\left[
\begin{array}{c}
\bm{0}_{1 \times N} \\ 
\bm{0}_{5d \times N} \\ 
\attn^{(2)}(\mX^{(2)})  \\ 
\bm{0}_{12d \times N}
\end{array}
\right],
\end{equation}
where $\attn^{(2)}$ denote the self-attention layer of $\tf^{(2)}$ and \( \mathbf{0}_{m \times N} \) denote \( m \times N \) zero matrix. 
Let 
\[
x_0 \in \R, \quad
\vx_1 \in \R^{5d}, \quad
\vx_2 \in \R^{8}, \quad
\vx_3 \in \R^{12d}
\]
denote the token-wise input space.
Define $\ff: \R^{17d+9} \to \R^{17d+9}$, with the impulse function in~\Cref{lemma:impulse}, as:
\begin{align}
\ff\left(
\left[
\begin{array}{c}
x_0 \\
\vx_1  \\ 
\vx_2  \\ 
\vx_3  \\ 
\end{array}
\right]
\right)
= 
\left(
\left[
\begin{array}{c}
1 \\
\ff^{(1)}(\vx_1)  \\ 
\ff^{(2)}(\vx_2) + \bm{\mathrm{impulse}}_{(\delta^{-1}-1)}\big(\mathcal{L}'(\vx_1), x_0 \big)
\\ 
\ff^{(3)}(\vx_3) + \bm{\mathrm{impulse}}_{(\delta^{-1}+N)}\big(\mathcal{L}''( \vx_2 ), x_0 \big) \\ 
\end{array}
\right]
\right)
\end{align}
where $\ff^{(2)}$ denotes the feed-forward layer of $\tf^{(2)}$, two linear maps $\mathcal{L}': \R^{5d} \to \R^8$ and $\mathcal{L}'': \R^8 \to \R^{12d}$ are defined respectively as follows:
\begin{align}
    \mathcal{L}'(\vx) &\coloneq \mathcal{L}^{(2)}_{1} 
    \left( \left[
    \begin{array}{ccc}
        \mathcal{L}^{(1)}_{2}(\vx)  \\
        \vx_{5d}  \\
    \end{array}\right] \right)\\
    \mathcal{L}''(\vx) &\coloneq \mathcal{L}^{(3)}_{1} \Big( \mathcal{L}^{(3)}_{0} \big((\vx_2)_1, \mathcal{L}^{(2)}_{2}(\vx_2)\big) \Big)
\end{align}
and $\bm{\mathrm{impulse}}$ refers to the dimension-wise application of $\mathrm{impulse}$ function. Note that $x_0$ serves the role of a counter. 

Each step should always be zero or set to an appropriate initial value at the beginning of the corresponding loop iteration. If the bias term causes it to deviate from this value before the iteration starts, the offset can be subtracted in advance to compensate.

As shown in \Cref{lemma:impulse}, the impulse function requires $2$ ReLU functions per dimension and $2$ ReLU functions for the corresponding loop iteration. Since the total dimension of $\vx_2$ and $\vx_3$ is $12d+8$, this results in an additional width size of $24d + 16 + 4 = 24d + 20$. Since the width size is $7d$ for $\ff^{(1)}$, $5$ for $\tf^{(2)}$, and $18d$ for $\ff^{(3)}$, resulting in a total width size of $49d + 25$.

Define two affine linear maps $\mathcal{L}_1:\R^d\to\R^{17d+9}$ and $\mathcal{L}_2:\R^{17d+9}\to\R^d$ via
\begin{equation}
    \mathcal{L}_1(\vx) = {(0, \mathcal{L}^{(1)}_1(\vx), {\zero_{12d+8}}^\top)}^\top, \quad
     \mathcal{L}_2\big({(x_0, \vx_1 , \vx_2,\vx_3)}^\top\big) = \mathcal{L}^{(3)}_2(\vx_3).
\end{equation}
Thus, the network, consisting of three steps, is defined as:
\begin{equation}
\begin{split}
    \tilde{f}(\mX) &\coloneq \bm{\mathcal{L}}_2\circ \mathrm{TF}^{\circ r}\circ \bm{\mathcal{L}}_1(\mX)
\end{split}
\end{equation}
where $r = \delta^{-1} + N + 2\delta^{-(N+1)d}$ and $\mathrm{TF}: \R^{(17d+9)\times N} \to \R^{(17d+9)\times N}$ consists of $\attn$ and $\ff$.

\paragraph{Deriving Approximation Error}
Generally, the following inequality holds.
\begin{equation}
\sum_{i=1}^{n} x_i^p \leq \left( \sum_{i=1}^{n} x_i \right)^p \quad \text{for } x_i \geq 0 \text{ and } p \geq 1.
\end{equation}
Substituting $x_i = \| \bar{f}(\mX) - f(\mX) \|_p$ into the above inequality results in
\begin{equation}\label{ineq:norm_app_0}
    \| \bar{f} - f \|_{L^p} = \Big (\int \norm{\bar{f}(\mX) - f(\mX)}_p^p d\mX \Big )^{1/p} \leq  \int \| \bar{f}(\mX) - f(\mX) \|_p d\mX. 
\end{equation}
Also, generally, the following inequality holds, for $\vx \in \mathbb{R}^m$, 
\begin{equation}\label{ineq:main_norm}
    \max_i |\vx_i| \leq \|\vx\|_p \leq m^{\frac{1}{p}} \max_i |\vx_i|.
\end{equation}
With the triangle inequality, we can bound the approximation error as
\begin{align}
     \| \tilde{f} - f \|_{L^p}
     &\leq \int \| \tilde{f}(\mX) - f(\mX) \|_p d\mX \quad \text{because of Eq.~\ref{ineq:norm_app_0}} \\ 
     &\leq 
     \int\| \tilde{f}(\mX) - \bar{f}(\mX) \|_p d\mX + \int\| \bar{f}(\mX) - f(\mX) \|_p d\mX  \\
     &\leq {(Nd)}^{\frac{1}{p}} \max_{k'\in \calK} | \tilde{\vy}_{k'} - \vy_{k'} | + \int\| \bar{f}(\mX) - f(\mX) \|_p d\mX + \mathcal{O}(N^{2/p} \delta^{d/p}) \quad \text{because of Eq.~\ref{ineq:main_norm}} \\
     &\leq {(Nd)}^{\frac{1}{p}} \max_{k'\in \calK} | \tilde{\vy}_{k'} - \vy_{k'} | + \omega_f(\delta \sqrt{Nd}) + \mathcal{O}(N^{2/p} \delta^{d/p}) \quad \text{because of Eq.~\ref{eq:cube-err}} \\
     &\leq (Nd)^{\frac{1}{p}} \big( \omega^{\text{tok}}_f(\delta\sqrt{d}) + \omega^{\text{cont}}_f(\delta\sqrt{Nd}) \big) + \omega_f(\delta \sqrt{Nd}) + \mathcal{O}(N^{2/p} \delta^{d/p}) \quad \text{because of Eq.~\ref{ineq:i_error}}.
\end{align}
Then, $\delta$ is expressed in terms of the number of loops as:
\begin{align}
   r = \delta^{-1} + N + 2\delta^{-(N+1)d} &\Leftrightarrow 
   \delta^{-1} + 2\delta^{-(N+1)d} = r - N  \\
    &\Rightarrow \delta^{-1} \cdot 2\delta^{-(N+1)d} \ge r - N  \\
    &\Leftrightarrow \delta \leq {(\frac{r-N}{2})}^{-1/ \big((N+1)d+1\big)}.
\end{align}
Thus, we have completed the proof of \Cref{thm:rate}.
\end{proof}

\subsection{Proof of Theorem \ref{thm:universal}}
Extending the construction in \Cref{thm:rate}, we then account for the boundedness of the weights and bit complexity.

\begin{proof}
Due to the use of bounded weights to approximate discontinuous functions, there inevitably exist regions where quantization errors arise in Step 1 of our construction. We define these regions, for \( 0 < \epsilon < \delta \), as
\begin{equation}
\Omega([0,1]^{d \times N}, \delta^{-1}, \epsilon) \coloneqq \left\{ \mX \in [0,1]^{d \times N} : \exists\, i, n \quad \text{such that} \quad \mX_{i,n} \in \bigcup_{k=1}^{\delta^{-1}} \left( k\delta - \epsilon, k\delta \right) \right\}.
\end{equation}
That is, \(\Omega\) consists of all inputs for which at least one coordinate \(\mX_{i,n}\) lies within an \(\epsilon\)-neighborhood of a quantization discontinuity point \(k\delta\) for some \(k \in \{1, \dotsc, \delta^{-1}\}\).

Outside the region \(\Omega\), the quantization function is piecewise constant and can be precisely approximated using bounded weights. By the construction in Lemma~\ref{lemma:quantize}, the maximum magnitude of weights required is proportional to \( 1/\epsilon \).
We now estimate the Lebesgue measure of \(\Omega\). For each coordinate \((i,n)\), the set
\[
\bigcup_{k=1}^{\delta^{-1}} (k\delta - \epsilon, k\delta)
\]
is a union of \(\delta^{-1}\) intervals, each of length \(\epsilon\), so the total measure of this set is \(\delta^{-1} \epsilon\).
Since the input \(\mX\) has \(d \times N\) coordinates, and since we consider the event that at least one coordinate lies in a bad region, we apply the union bound for measures to obtain
\(
\mathrm{meas}(\Omega) \leq dN \times \delta^{-1} \epsilon.
\)
Substituting the bound on \(\mathrm{meas}(\Omega)\), and using the fact that the maximum magnitude of the weights \(M\) satisfies \(M = 1/\epsilon\), we have
\(\mathrm{meas}(\Omega) \leq dN \times \delta^{-1} M^{-1}.\)
Consequently, the \(L^p\)-norm of the quantization error is bounded by \(\mathcal{O}\left(\left( (M\delta)^{-1} dN \right)^{1/p}\right).\)

When replacing hardmax with softmax, it is required that the error in step 2 remains sufficiently small so that it does not affect step 3. Specifically, a step function is used, in \Cref{lemma:memorize} for step 3, to map the index, defined for $\epsilon>0$ as
\begin{equation}
    \mathrm{step}_\epsilon(x) = 
    \begin{cases}
        1 \quad & \text{if}\ x \geq 0,\\
        0 \quad & \text{if}\ x < 0-\epsilon,
    \end{cases}
\end{equation}
The construction of \Cref{lemma:memorize} use this function for indices $k\in \calK$ obtained from step 2 as $\mathrm{step}_\epsilon(k-i)$ for $i=0, 1 \cdots K-1$. Thus, the error in step 2 does not affect step 3 if the perturbed indices, denoted by $\tilde{k}$, satisfies
\begin{equation}\label{eq:soft_step}
\mathrm{step}_\epsilon(\tilde{k} + 1 -i) = \mathrm{step}_\epsilon(k-i) \quad \text{for all }  i=0, 1 \cdots K-1.
\end{equation}

To estimate the error introduced by replacing hardmax with softmax in step 2, we revisit the construction of \Cref{lemma:contextmap}. Specifically, we extract and reformulate the key components necessary for this estimation. In particular, we consider a simplified form of the attention computation in Eq.~\ref{eq:attn_2}, denoted by \( g_H: \mathbb{R}^N \times \mathbb{R}^N \to \mathbb{R} \), which is defined as  
\begin{equation}
    g_H(\vv, \va) \coloneq \vv_{\arg\max_{i} \va_i}
\end{equation}  
When hardmax in Eq.~\ref{eq:attn_2} is replaced with softmax, the function can be expressed as
\begin{equation}
    g_S(\vv, \va, \lambda) \coloneq \sum_{i=1}^{N} \frac{\exp(\lambda \va_i)}{\sum_{j=1}^{N} \exp(\lambda\va_j)} \vv_i,
\end{equation}
where $\lambda > 0$. Note that $ \lim_{\lambda\to \infty} g_S(\vv, \va, \lambda)= g_H(\vv, \va)$. According to the construction of \Cref{lemma:contextmap}, such as Eq.~\ref{eq:vxk}, the domains, denoted by $\vv \in \sL_{\delta}$ and $\va \in \sA_\delta$, are restricted on 
\begin{gather}
\sL_{\delta} \coloneq \{\vz \in \{0, 1, \dotsc, \delta^{-d} - 1\}^N \mid \vz_{i} \neq \vz_{j} \text{ for all } i \neq j\}, \\
\sA_\delta \coloneq \left\{ \va \in \mathbb{R}^N \;\middle|\;
\va_i \in \{0,1,\dots, \delta^{-d}-1\} \text{ or } \va_i < 0, \quad
\text{if } \va_i, \va_j \geq 0, \text{ then } \va_i \neq \va_j \text{ for all } i \neq j, \quad \exists i \text{ s.t. } a_i > 0
\right\}.
\end{gather}
We impose the following two additional assumptions on \(\vv \in \sL_{\delta}\) and
\(\va \in \sA_\delta\)
\begin{equation}
    n \;=\; \arg\max_{1 \le i \le N} \va_i,
\qquad
v_n \;=\; \max_{1 \le i \le N} v_i.
\end{equation}

Under these assumptions, for any finite \(\lambda>0\) we have
\begin{equation}
0 < g_S(\vv, \va, \lambda) < g_H(\vv, \va)
\end{equation}
and
\begin{equation}
    g_S(\vv, \va, \lambda) \ge \frac{\exp(\lambda \va_n)}{\sum_{j=1}^{N} \exp(\lambda\va_j)} \vv_n,
\end{equation}
Thus we have
\begin{align}
     g_H(\vv, \va) - g_S(\vv, \va, \lambda) &= \vv_n - g_S(\vv, \va, \lambda) \\
     & \le \frac{\sum_{j=1}^{N} \exp(\lambda\va_j) - \exp(\lambda \va_n)}{\sum_{j=1}^{N} \exp(\lambda\va_j)} \vv_n  \\
     & = \frac{\sum_{j\ne n} \exp(\lambda\va_j)}{\sum_{j\ne n} \exp(\lambda\va_j) +\exp(\lambda \va_n) } \vv_n \\
     & \le \frac{\sum_{j\ne n} \exp(\lambda\va_j)}{\exp(\lambda \va_n) } \vv_n \\
     & = \vv_n \sum_{j\ne n} \exp\big(\lambda(\va_j - \va_n)\big) \\
     & \le \vv_n N\exp(-\lambda) \quad \text{because } \va_j - \va_n \le -1 \\
     & \le \delta^{-d} N \exp(-\lambda).
\end{align}
Thus, if we aim to bound the error within $\epsilon > 0$, $\lambda$ must satisfies
\begin{equation}
     \delta^{-d} N \exp(-\lambda) < \epsilon \Leftrightarrow \lambda > \log\Big(\frac{\delta^{-d} N}{\epsilon}\Big).
\end{equation}\label{eq:lambda}
From Eq.~\ref{eq:context_id}, the error of contextual ID $k - \tilde{k}$ can be bounded in terms of the error of $\epsilon = g_H(\vv, \va) - g_S(\vv, \va, \lambda)$ as:
\begin{equation}
    k - \tilde{k} \le (\vu_{(\delta^{-d})})^\top (\epsilon \bm{1}_N) \le \epsilon N\delta^{-(N-1)d}
\end{equation}
where $\vu_{(\delta^{-d})} \coloneqq (\delta^{-(N-1)d}, \dotsc, \delta^{-d}, 1)^\top$ and $\bm{1}_N \in \R^{N}$ denotes the all-ones vector.

Since Eq.~\ref{eq:soft_step} holds if $0 < k - \tilde{k} < 1$, it follows that
\begin{equation}
    \lambda > \log\Big(\big(\delta^{-d}N\big)\cdot(N\delta^{-(N-1)d})\Big) \Rightarrow \epsilon N\delta^{-(N-1)d} < 1,
\end{equation}
and Eq.~\ref{eq:soft_step} holds. Thus, the bit complexity of $\lambda$ can be bounded by 
\begin{equation}
    \mathcal{O}\big(\log\log(\delta^{-Nd} N^2)\big),
\end{equation}
while ensuring that no error occurs in step 3 when using the softmax function instead of the hardmax function.

The bit complexity at each step of the computation can be analyzed as follows.  
In Step~1 and Step~2, the bit complexity is bounded by \( \mathcal{O}(\log \delta^{-1}) \), reflecting the cost of maintaining precision within error \( \delta \).  
In contrast, Step~3 incurs a significantly higher cost, with a bit complexity bounded by \( \mathcal{O}(2\delta^{-(N+1)d}) \), due to the need to evaluate higher-order terms accurately.

As a result, the overall bit complexity of the Looped Transformer is dominated by Step~3 and can be bounded by
\begin{equation}
    \mathcal{O}(\max\{\log\log(\delta^{-Nd} N^2), 2\delta^{-(N+1)d}\}) = \mathcal{O}(\delta^{-(N+1)d}) \qquad (\delta \to 0).
\end{equation}
With \Cref{thm:rate}, the proof of \Cref{thm:universal} is completed.
\end{proof}

\subsection{Approximating Discontinuous Piecewise Functions}
We define three utility functions using ReLU activations. Since the target function is discontinuous, there are negligible `trifling regions' introduced by the bounded weights of the networks.
\begin{proposition}[Rectangular function]
\label{lemma:rect}
Given $t_1, t_2 \in \R$, there exist four $\mathrm{ReLU}$ functions that can approximate the following function, denote by $\mathrm{rect}_t:\R \to \R$, for $0 <\epsilon<t_2-t_1$, such that:
\begin{equation}
    \mathrm{rect}_{(t_1, t_2, \epsilon)}(x) = 
    \begin{cases}
        1 \quad & \text{if}\ x \in [t_1,t_2-\epsilon],\\
        0 \quad & \text{if}\ x \in (-\infty, t_1-\epsilon]\cup[t_2, \infty),
    \end{cases}
\end{equation}
which is represented by
\begin{equation}
    \mathrm{rect}_{(t_1, t_2, \epsilon)}(x) = \relu\big(\tfrac{x-t_1+\epsilon}{\epsilon}\big)
				- \relu\big(\tfrac{x-t_1}{\epsilon}\big)
				+\relu\big(\tfrac{-x+t_2}{\epsilon}\big)
				-\relu\big(\tfrac{-x+t_2-\epsilon}{\epsilon}\big)-1.
\end{equation}
\end{proposition}
Note that maximum magnitude of the weights is $1/\epsilon$ and  `trifling regions' are $(t-\epsilon, t)\cup(t+1-\epsilon, t-1)$.
\begin{proposition}[Step function]
\label{lemma:step}
There exist four $\mathrm{ReLU}$ functions that can approximate the following function, denote by $\mathrm{step}:\R \to \R$, for $\epsilon>0$, such that:
\begin{equation}
    \mathrm{step}_\epsilon(x) = 
    \begin{cases}
        1 \quad & \text{if}\ x \geq 0,\\
        0 \quad & \text{if}\ x < 0-\epsilon,
    \end{cases}
\end{equation}
which is represented via
\begin{equation}
    \mathrm{step}_\epsilon(x) = \relu\big(\tfrac{x}{\epsilon}+1\big)-\relu\big(\tfrac{x}{\epsilon}\big).
\end{equation}
\end{proposition}
\begin{proposition}[Impulse function]
\label{lemma:impulse}
Given $\theta \in \N$, there exist four $\mathrm{ReLU}$ functions that can approximate the following function, denote by $\mathrm{impulse}_\theta:\R \times \N \to \R$ for $x \in [-M,M]$ and $t \in \N$, such that:
\begin{equation}
    \mathrm{impulse}_\theta(x,t)=\begin{cases}
        x \quad & \tn{if}\  t = \theta,  \\
        0 \quad  & \tn{otherwise}.
    \end{cases}
\end{equation} 
which is represented via
\begin{equation}
\begin{split}\label{eq:def:g:uvt}
    \mathrm{impulse}_\theta(x,t) \coloneqq&\relu\big(x+2M(t-\theta+1/2)\big)-2M\relu(t-\theta+1/2)\\ &-\relu\big(x+2M(t-\theta-1/2)\big)+2M\relu(t-\theta-1/2).   
\end{split}
\end{equation}
\end{proposition}

\subsection{Step 1. Token-wise Quantization}
\label{appendix:quant}

\begin{figure}[ht]
\begin{center}
    \includegraphics[width=0.8\linewidth]{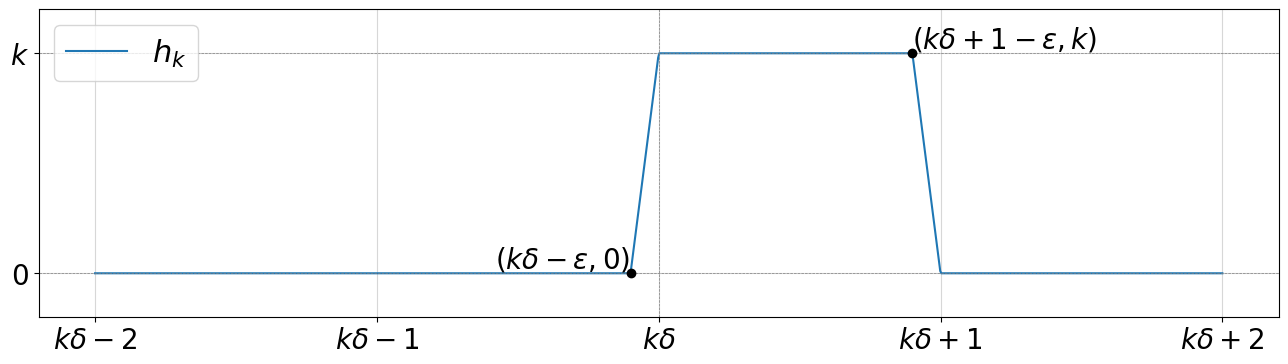}
\end{center}
\caption{An illustration of $h_k(x)$.}\label{fig:hk}
\end{figure}

We aim to construct quantization function $\mathbf{g}: [0,1]^d \to \{0, 1, \dotsc, \delta^{-1}\}^d$, for $\epsilon<\delta$, for each dimension as
\begin{equation}
    \mathbf{g}(\vx) = {(g(\vx_1), g(\vx_2), \dotsc, g(\vx_d))}^\top, \quad
    \text{where} \ g(x) = k \quad \text{if}\ x \in [k\delta,(k+1)\delta-\epsilon] \ \text{for}\ k = 0, \dotsc, \delta^{-1}-1.
\end{equation}
This function $g: \R \to \R$ can be expressed as
\begin{equation}
    g(x) = \sum_{i=0}^{n-1} i \cdot \mathrm{rect}_{(i\delta, (i+1)\delta), \epsilon)}(x)
\end{equation}
for any $n\in\N$ and $x\in\R$. 
The illustration of $h_k(x) \coloneq k\cdot\mathrm{rect}_{(k\delta, (k+1)\delta))}(x)$ is shown in Fig~\ref{fig:hk}.
The key idea is that $h_k(x)$ can be represented with a single function $h$ in the form of $h\big(kx,k^2,k\big)$.
\Cref{lemma:quantize} implement $h\big(kx,k^2,k\big)$ with a feed-forward layer and perform the summation through a skip connection.

\begin{lemma}\label{lemma:quantize}
Given any $\delta^{-1} \in \N$ and $\vx \in \mathbb{R}^d$, there exists a feed-forward layer $\ff : \mathbb{R}^{5d} \to \mathbb{R}^{5d}$ of width size $q = 7d$ with the maximum magnitude of the weights $1/\epsilon$, and two affine linear maps $\mathcal{L}_1 : \mathbb{R}^d \to \mathbb{R}^{5d}$ and $\mathcal{L}_2 : \mathbb{R}^{5d} \to \mathbb{R}^d$ \st
\begin{equation}
    \left(\mathcal{L}_2\circ \big(\mathrm{id} + \ff\big)^{\circ (\delta^{-1} - 1)}\circ\mathcal{L}_1(\vx)\right)_{i} = k \quad \text{if}\ x \in [k\delta,(k+1)\delta-\epsilon], \quad \text{for}\ k = 0, \dotsc, \delta^{-1}-1
\end{equation}
for any $i = 1, 2, \dotsc, d$.
\end{lemma}
\begin{proof}
On the basis of \cref{lemma:rect}, define function $h_k(x)=k \cdot \mathrm{rect}_k(x)$ via
\begin{equation}
\begin{split}
    h_k(x) \coloneqq \relu\big(\tfrac{k}{\epsilon}(x-k\delta+\epsilon)\big)
				&- \relu\big(\tfrac{k}{\epsilon}(x-k\delta)\big)
				+\relu\big(\tfrac{k}{\epsilon}(-x+k\delta+1)\big)\\
				&-\relu\big(\tfrac{k}{\epsilon}(-x+k\delta+1-\epsilon)\big)-k,   
\end{split}
\end{equation}
which satisfies
\begin{equation}
    h_k(x)= 
    \begin{cases}
        k \quad & \text{if}\ x \in [k{\delta}, (k+1)\delta-\epsilon],\\
        0 \quad & \text{if}\ x \in (-\infty, k{\delta}-\epsilon]\cup[(k+1)\delta, \infty),
    \end{cases}
\end{equation}
For any $x\in [k{\delta}, (k+1)\delta-\epsilon]$ for $k=0,1, \dotsc, \delta^{-1} -1$, we have
\begin{equation}
	\sum_{i=0}^{\delta^{-1}-1}h_i(x)= h_k(x)=k.
\end{equation}
Define function $h:\R^3 \to \R$ to represent $h_k$ via
\begin{equation}
\begin{split}
h(kx,k^2,k)\coloneqq \relu\big(\tfrac{kx}{\epsilon}-\tfrac{k^2\delta}{\epsilon}+k\big)
&-\relu\big(\tfrac{kx}{\epsilon}-\tfrac{k^2\delta}{\epsilon}\big)
+\relu\big(-\tfrac{kx}{\epsilon}+\tfrac{k^2\delta}{\epsilon}+\tfrac{k}{\epsilon}\big) \\
&-\relu\big(-\tfrac{kx}{\epsilon}+\tfrac{k^2\delta}{\epsilon}+k\tfrac{1-\epsilon}{\epsilon})-\relu(k) =h_k(x).
\end{split}
\end{equation}
Define $\boldsymbol\xi_k$ as
\begin{equation}
  \boldsymbol\xi_k = {\Big(kx,\  k^2,\  k,\  x,\  \sum_{i=0}^{k-1}h_i(x)\Big)^\top}.
\end{equation}
Then, construct a feed-forward layer $\ff: \R^5 \to \R^5$ with a skip connection such that
\begin{align}
    \big(\mathrm{id} + \ff\big) (\boldsymbol\xi_k) 
    &= \big(\mathrm{id} + \ff\big) \bigg( (kx,\  k^2,\  k,\  x,\  \sum_{i=0}^{k-1}h_i(x))^\top \bigg) \\
    &= \Big( (k+1)x,\  (k+1)^2,\  k+1,\  x,\  \sum_{i=0}^{k}h_i(x) \Big)^\top \\
    &= \boldsymbol\xi_{k+1}.
\end{align}
via
\begin{align}
\big(\mathrm{id} + \ff\big)
\left(
\left[
\begin{array}{c}
kx  \\
k^2  \\
k  \\
x  \\
\sum_{i=0}^{k-1}h_i(x)  \\
\end{array}
\right]
\right)
&=
\left[
\begin{array}{c}
kx  \\
k^2  \\
k  \\
x  \\
\sum_{i=0}^{k-1}h_i(x)  \\
\end{array}
\right] + \notag \\
\begin{bmatrix}
1 & 0 & 0 & 0 & 0 & 0& 0\\
0 & 1 & 0 & 0 & 0 & 0& 0\\
0 & 0 & 0 & 0 & 0 & 0& 0\\
0 & 0 & 0 & 0 & 0 & 0& 0\\
0 & 0 & 1 & -1 & 1 & -1 & -1\\
\end{bmatrix}
&\relu \left(
\begin{bmatrix}
0 & 0 & 0 & 1 & 0 \\
0 & 0 & 2 & 0 & 0 \\
\tfrac{1}{\epsilon} & -\tfrac{\delta}{\epsilon} & 1 & 0 & 0 \\
\tfrac{1}{\epsilon} & -\tfrac{\delta}{\epsilon} & 0 & 0 & 0 \\
-\tfrac{1}{\epsilon} & \tfrac{\delta}{\epsilon} & \tfrac{1}{\epsilon} & 0 & 0 \\
-\tfrac{1}{\epsilon} & \tfrac{\delta}{\epsilon} & \tfrac{1-\epsilon}{\epsilon} & 0 & 0 \\
0 & 0 & 1 & 0 & 0 \\
\end{bmatrix}
\left[\begin{array}{c}
kx  \\
k^2  \\
k  \\
x  \\
\sum_{i=0}^{k-1}h_i(x)  \\
\end{array}\right]\right)
+
\left[\begin{array}{c}
0  \\
1  \\
1  \\
0  \\
0  \\
\end{array}
\right]
\\
& = 
\left[\begin{array}{c}
kx  \\
k^2  \\
k  \\
x  \\
\sum_{i=0}^{k-1}h_i(x)  \\
\end{array}
\right]
+ 
\left[
\begin{array}{c}
x  \\
2k + 1  \\
1  \\
0  \\
h_k(x) \\
\end{array}
\right]\\
&= 
\left[
\begin{array}{c}
kx + x  \\
k^2 + 2k + 1  \\
k + 1  \\
x  \\
\sum_{i=0}^{k-1}h_i(x) + h_k(x) \\
\end{array}
\right]\\
&=
\left[
\begin{array}{c}
(k+1)x  \\
(k+1)^2  \\
k+1  \\
x  \\
\sum_{i=0}^{k}h_i(x)  \\
\end{array}
\right].
\end{align}
Then, define two affine linear maps $\mathcal{L}_1 : \mathbb{R}^1 \to \mathbb{R}^{5}$ and $\mathcal{L}_2 : \mathbb{R}^{5} \to \mathbb{R}^1$ by
\begin{equation}
    \mathcal{L}_1(x) \coloneqq (0, 0, 0, x, 0), \quad
    \mathcal{L}_2(x_1, x_2, x_3, x_4, x_5) \coloneqq x_5.
\end{equation}
Thus, we have
\begin{align}
    \mathcal{L}_2\circ \big(\mathrm{id} + \ff\big)^{\circ (\delta^{-1} - 1)}\circ\mathcal{L}_1(x)
    &= \mathcal{L}_2\circ \big(\mathrm{id} + \ff\big)^{\circ (\delta^{-1} - 1)}(\boldsymbol \xi_{0})\\
    &= \mathcal{L}_2 (\boldsymbol \xi_{\delta^{-1}})\\
    &=\sum_{i=0}^{\delta^{-1}-1}h_i(x).
\end{align}
For $d$-dimensional inputs, we need $d$-times more parameters.
\end{proof}

\subsection{Step 2. Contextual Mapping}
\label{appendix:contextmap}
The network takes token IDs as inputs, denoted by $\vz \in \{0, 1, \dotsc, \delta^{-d} - 1\}^N$.
We only consider cases where all token IDs are distinct. The network maps token IDs into a sequence ID using the inner product with the vector $\vu \in \R^N$ defined as $\vu \coloneqq (\delta^{-(N-1)d}, \delta^{-(N-2)d}, \dotsc, \delta^{-d}, 1)^\top$ \ie
\begin{equation}
    \mathrm{s}(\vz) \coloneq \vu^\top \vz.
\end{equation}
Due to permutation equivariance, we can assume without loss of generality that elements of $\vz \in \sL_{\delta}$ are ordered, \ie, $\vz_1 > \vz_2 > \cdots > \vz_N$. Then the map $\mathrm{s}$ satisfies
\begin{equation}
    |\vu^\top \vz- \vu^\top \vz'| > 1,  \quad \text{if }  \vz \ne \vz'.
\end{equation}
In other words, $\mathrm{s}$ represent $\vz$ in $\delta^{-d}$-base system.
The network computes $\vu^\top \vz$ in the form of $\sum_{i=1}^{N} \delta^{-(N-i)d} \vz_i$. 
The network computes $\vs^{(k)} \coloneq \sum_{i=1}^{k} \delta^{-(k-i)d} \vz_i$ in each loop, and after $N$-loops, it outputs $\vs^{(N)} =\vu^\top \vz$.
To implement this, the self-attention layer selects $\vz_k$ in the $k$-th loop iteration. We design the key, query, and value weights to select the maximum token ID. The feed-forward layer post-processes the token ID in such a way that if it is selected, then it is replaced with a negative value to prevent selection in subsequent iterations, \ie, the post-processed token IDs for the $k$-th loop are 
\begin{align*}
    \vz^{(k)}_i &= z \quad \st \quad
    \begin{cases*}
        z < 0 \quad \tn{if} \quad i \leq k, \\
        z = \vz_i \quad \tn{otherwise}.
    \end{cases*}
\end{align*}

We focus on self-attention layers that employ the hardmax function.
\begin{lemma}\label{lemma:contextmap}
Consider the set of distinct indices corresponding to the \( d \)-dimensional \( \delta \)-discretized regions of \( N \) tokens, \ie
\begin{equation}
\sL_{\delta} \coloneq \{\vz \in \{0, 1, \dotsc, \delta^{-d} - 1\}^N \mid \vz_{i} \neq \vz_{j} \text{ for all } i \neq j\}.
\end{equation}
There exists a function \(\mathrm{s}: \mathbb{R}^{N} \to \mathbb{R}\) composed of Transformer block \(\mathrm{TF}: \mathbb{R}^{5 \times N} \to \mathbb{R}^{5 \times N}\) with the hardmax function, single head, head size $s=1$, and width size $q=3$, and two affine linear maps \(\mathcal{L}_1: \mathbb{R} \to \mathbb{R}^5\) and \(\mathcal{L}_2: \mathbb{R}^5 \to \mathbb{R}\), such that
\[
\bm{\mathcal{L}}_2 \circ {\tf}^{\circ N} \circ \bm{\mathcal{L}}_1(\vz^\top) = (\vu^\top\vz) \cdot \mathbf{1}_N^\top,
\]
for any $\vz \in \sL_{\delta}$, where $\vu \coloneqq (\delta^{-(N-1)d}, \delta^{-(N-2)d}, \dotsc, \delta^{-d}, 1)^\top$.
\end{lemma}

\begin{proof}
Due to permutation equivariance, we can assume without loss of generality that elements of $\vz \in \sL_{\delta}$ are ordered, \ie, $\vz_1 > \vz_2 > \cdots > \vz_N$.
Define $\vu \in \R^N$ as $\vu \coloneqq (\delta^{-(N-1)d}, \dotsc, \delta^{-d}, 1)^\top$, which satisfy
\begin{equation}
    |\vu^\top \vz- \vu^\top \vz'| > 1,  \quad \text{if }  \vz \ne \vz' \text{ for any } \vz, \vz' \in \sL_{\delta}.
\end{equation}
We construct Transformer block $\tf:\R^{5 \times N}\to \R^{5 \times N}$ with single head and head size $s=1$ such that, for any $\vz \in \sL_{\delta}$, 
\begin{equation}\label{eq:attn_2}
    \mathrm{TF}^{\circ N} \left( \left[
    \begin{array}{ccc}
    \vz^\top  \\
    \vz^\top  \\
    {\mathbf1}_N^\top \\
    {\mathbf0}_N^\top \\
    {\mathbf0}_N^\top \\
    \end{array}\right]\right) = 
    \left[
    \begin{array}{ccc}
    {\mathbf0}_N^\top \\
    {\mathbf0}_N^\top \\
    {\mathbf1}_N^\top \\
    {\mathbf0}_N^\top \\
    (\vu^\top \vz) \cdot \mathbf{1}_N^\top \\
    \end{array}\right].    
\end{equation}
where ${\mathbf0}_N, {\mathbf1}_N \in \R^N$ denote the all-zero and all-one vectors, respectively. For $\vz \in \sL_\delta$, define two series for $k=0, \cdots N$ as:
\begin{align}
    \vz^{(k)}_i &\coloneq z \quad \st \quad
    \begin{cases*}
        z < 0 \quad \tn{if} \quad i \leq k, \\
        z = \vz_i \quad \tn{otherwise},
    \end{cases*} \quad \text{for } i = 1, \dotsc, N, \quad \in \R^{N},
    \\
    \vs^{(k)} &\coloneq \sum_{i=0}^{k} \delta^{-(k-i)d} \vz_i \quad \in \R.
\end{align}
While \(\vz^{(k)}\) is not uniquely determined, any vector that satisfies the conditions is accepted as \(\vz^{(k)}\). The series $\vs^{(k)}$ satisfies
\begin{align}
    \vs^{(k+1)} &= \sum_{i=1}^{k+1} \delta^{-(k+1-i)d} \vz_i \\
    &= \big(\sum_{i=1}^{k+1-1} \delta^{-(k+1-i)d} \vz_i \big) + \vz_k \\
    &= \big(\sum_{i=0}^{k} \delta^{-d} \cdot \delta^{-(k-i)d} \vz_i \big) + \vz_{k+1}\\
    &= \delta^{-d} \cdot \vs^{k} + \vz_{k+1} \label{eq:vs},
\end{align}
for $k = 0, \ldots, N-1$. Note that $\vs^{(N)} = \vu^\top \vz$. Define a single-head self-attention $\attn:\R^{5 \times N}\to \R^{5 \times N}$ such that
\begin{equation}
    \attn \left( \left[
    \begin{array}{ccc}
        {\vv}^\top  \\
        {\va}^\top  \\
        {\mathbf1}_N^\top \\
        {*}^\top   \\
        {*}^\top   \\
    \end{array}\right] \right) 
    = \left[ \begin{array}{ccc}
        {\mathbf0}_N^\top \\
        {\mathbf0}_N^\top \\
        {\mathbf0}_N^\top \\
        (\vv_{\argmax_j {\va}_j}) \cdot \mathbf{1}_N^\top \\
        {\mathbf0}_N^\top \\
    \end{array} \right],
\end{equation}
where $\vv,\va \in \R^{N}$, and $*\in \R^{N}$ denotes arbitrary vectors, via the weight parameters 
\begin{equation}
    \mW_O = \left[ \begin{array}{ccc}
        0 \\
        0 \\
        0 \\
        1 \\
        0 \\
        \end{array}\right], \quad
    \mW_V = \left[ \begin{array}{ccccc}
        1 & 0 & 0 & 0 & 0
    \end{array} \right], \quad
    \mW_K = \left[ \begin{array}{ccccc}
        0 & 1 & 0 & 0 & 0
    \end{array} \right], \quad
    \mW_Q = \left[ \begin{array}{ccccc}
        0 & 0 & 1 & 0 & 0
    \end{array} \right].
\end{equation}
Define $\ff:\R^5\to \R^5$ of width size $q=3$ via:
\begin{align}
    \ff \left( \left[ 
    \begin{array}{ccc}
        x_1 \\
        x_2 \\
        x_3 \\
        x_4 \\
        x_5 \\
    \end{array} \right] \right) 
    & =
    \begin{bmatrix}
    0 & 0 & 0\\
    -M & 0 & 0\\
    0 & 0 & 0\\
    0 & -1 & 0\\
    0 & 1 & \delta^{-d} - 1 \\
    \end{bmatrix} \relu \left(
    \begin{bmatrix}
    0 & 1 &0 &-1 & 0\\
    0 & 0 &0 &1 & 0\\
    0 & 0 &0 &0 & 1\\
    \end{bmatrix}
   \left[ 
    \begin{array}{ccc}
        x_1 \\
        x_2 \\
        x_3 \\
        x_4 \\
        x_5 \\
    \end{array} \right]
    + \left[ 
    \begin{array}{ccc}
        0 \\
        \epsilon \\
        0 \\
        0 \\
        0 \\
    \end{array} \right]\right) 
    \\
    & = \left[\begin{array}{ccc}
        0 \\
        - M \sigma_R(x_2-x_4+\epsilon)  \\
        0 \\
        - \sigma_R(x_4)  \\
        (\delta^{-d} - 1) \sigma_R(x_5) + \sigma_R(x_4) \\
    \end{array}\right],
\end{align}
where $0 < \epsilon < 1$ and $M > \frac{\delta^{-d}-1}{\epsilon}$. 

For $x_2 \in \{ 0, 1, \dots, \delta^{-d} - 1 \} \cup \{ x \mid x \leq 0 \}$ and $x_4\in \{ 0, 1, \dots, \delta^{-d} - 1 \}$ with $ x_4 \ge x_2$, we have
\begin{equation}
    x_2 - M \sigma_R(x_2-x_4+\epsilon) = z, \quad \st \quad
    \begin{cases*}
    z = x_1 \quad \tn{if} \quad x_4 > x_2, \\
    z < 0 \quad \tn{if} \quad x_4 = x_2.
    \end{cases*}
\end{equation}
This post-processes the token ID in such a way that if it is selected, then it is replaced with a negative value \ie
\begin{equation}\label{eq:vxk}
    \begin{split}
        \vz^{(k)}_i - M \sigma_R(\vz^{(k)}_i-\vz_k+\epsilon)
        & = z \quad \st \quad
        \begin{cases*}
            z < 0 \quad \tn{if} \quad i \leq k + 1, \\
            z = \vz_i \quad \tn{otherwise},
        \end{cases*} \\
        &= \vz^{(k+1)}_i \quad \text{for } i = 1, \dotsc, N.
    \end{split} 
\end{equation}
We confirm that the Transformer block $\tf:\R^{5\times N}\to \R^{5\times N}$, composed of $\attn$ and $\ff$, satisfies, for $k = 0, \ldots, N-1$,
\begin{align}
    \tf \left( \left[\begin{array}{ccc}
    \vz^\top \\
    (\vz^{(k)})^\top  \\
    {\mathbf1}_n^\top \\
    {\mathbf0}_n^\top \\
    (\vs^{(k)})^\top \\
    \end{array}\right]\right)
    &= (\mathrm{id} + \bm\ff) \circ (\mathrm{id} + \attn) \left( \left[\begin{array}{ccc}
    \vz^\top \\
    (\vz^{(k)})^\top  \\
    {\mathbf1}_n^\top \\
    {\mathbf0}_n^\top \\
    (\vs^{(k)})^\top \\
    \end{array}\right]\right) \\
    &= (\mathrm{id} + \bm\ff) \left(\left[\begin{array}{ccc}
    \vz^\top \\
    (\vz^{(k)})^\top  \\
    {\mathbf1}_n^\top \\
    \vz_{k+1} \cdot \mathbf{1}_N^\top \\
    (\vs^{(k)})^\top \\
    \end{array}\right]\right) \\
    &= \left[\begin{array}{ccc}
    \vz^\top \\
    (\vz^{(k)})^\top  \\
    {\mathbf1}_n^\top \\
    \vz_{k+1} \cdot \mathbf{1}_N^\top \\
    (\vs^{(k)})\mathbf{1}_N^\top  \\
    \end{array}\right] +
    \left[\begin{array}{ccc}
    {\mathbf0}_n^\top \\
    - M \bm{\sigma}_R\big((\vz^{(k)})^\top -\vz_k \cdot \mathbf{1}_N^\top+\epsilon \mathbf{1}_N^\top \big)  \\
    {\mathbf0}_n^\top \\
    - \vz_{k+1} \cdot \mathbf{1}_N^\top \\
    (\delta^{-d} - 1) (\vs^{(k)})\mathbf{1}_N^\top  + \bm{\sigma}_R(\vz_{k+1}\mathbf{1}_N^\top) \\
    \end{array}\right] \\
    &= \left[\begin{array}{ccc}
    \vz^\top \\
    (\vz^{(k+1)})^\top \\
    {\mathbf1}_n^\top \\
    {\mathbf0}_n^\top \\
    \delta^{-d}(\vs^{(k)})\mathbf{1}_N^\top + \vz_{k+1} \mathbf{1}_N^\top \\
    \end{array}\right] \quad \text{because of Eq.~\ref{eq:vxk}}\\
    &= \left[\begin{array}{ccc}
    \vz^\top \\
    (\vz^{(k+1)})^\top  \\
    {\mathbf1}_n^\top \\
    {\mathbf0}_n^\top \\
    (\vs^{(k+1)})\mathbf{1}_N^\top \\
    \end{array}\right] \quad \text{because of Eq.~\ref{eq:vs}} .
\end{align}
Define two affine linear maps $\mathcal{L}_1 : \mathbb{R} \to \mathbb{R}^{5}$ and $\mathcal{L}_2 : \mathbb{R}^{5} \to \mathbb{R}$ via $\mathcal{L}_1(x) \coloneqq (x, x, 1, 0, 0)$ and $\mathcal{L}_2(x_1, x_2, x_3, x_4, x_5) \coloneqq x_5$ respectively. Thus, we have
\begin{align*}
    \bm{\mathcal{L}}_2 \circ \tf^{\circ N} \circ \bm{\mathcal{L}}_1 (\vz^\top)
    = (\vu^\top \vz) \cdot {\mathbf1}_N^\top.  
\end{align*}
\end{proof}

Combining the sequence ID with token ID, the network computes contextual mapping.
\begin{corollary}\label{corollary:contextmap}
There exists a Transformer block \(\mathrm{TF_2}: \mathbb{R}^{8 \times N} \to \mathbb{R}^{8 \times N}\) with the hardmax function, two heads, head size $s=1$, and width size $q=5$, and two affine linear maps \(\mathcal{L}_1: \mathbb{R}^2 \to \mathbb{R}^8\) and \(\mathcal{L}_2: \mathbb{R}^8 \to \mathbb{R}\) \st
\[
\bm{\mathcal{L}}_2 \circ {\tf^{(2)}}^{\circ N} \circ \bm{\mathcal{L}}_1 \left( \left[
    \begin{array}{ccc}
        {\vz}^\top  \\
        {\mZ_{d,:}}  \\
    \end{array}\right] \right)  = \big(2\vu^\top \vz - \mZ_{d,N}\big) \cdot \mathbf{1}_N^\top \quad \tn{for any $\vz\in\sL_{\delta}$.}
\]
\end{corollary}
\begin{proof}
Define a self-attention $\attn:\R^{8 \times N}\to \R^{8 \times N}$ such that
\begin{equation}
    \attn \left( \left[
    \begin{array}{ccc}
        {\vv}^\top  \\
        {\va}^\top  \\
        {\mathbf1}_N^\top \\
        {*}^\top  \\
        {*}^\top  \\
        {\mZ_{d,:}} \\
        {*}^\top  \\
        {*}^\top  \\
    \end{array}\right] \right) 
    = \left[ \begin{array}{ccc}
        {\mathbf0}_N^\top \\
        {\mathbf0}_N^\top \\
        {\mathbf0}_N^\top \\
        (\vv_{\argmax_j {\va}_j}) \cdot \mathbf{1}_N^\top \\
        {\mathbf0}_N^\top \\
        {\mathbf0}_N^\top \\
        \mZ_{d,N} \cdot \mathbf{1}_N^\top \\
        {\mathbf0}_N^\top \\
    \end{array} \right],
\end{equation}
where $\vv,\va \in \R^{N}$, and $*\in \R^{N}$ denotes arbitrary vectors, via the weight parameters
\begin{equation}
    \mW^1_O = \left[ \begin{array}{ccc}
        0 \\
        0 \\
        0 \\
        1 \\
        0 \\
        0 \\
        0 \\
        0 \\
        \end{array}\right], \quad
    \mW^1_V = \left[ \begin{array}{cccc}
        1 & 0 & \dots & 0 
    \end{array} \right], \quad
    \mW^1_K = \left[ \begin{array}{ccccc}
        0 & 1 & 0 & \dots & 0
    \end{array} \right], \quad
    \mW^1_Q = \left[ \begin{array}{cccccc}
        0 & 0 & 1 & 0 & \dots & 0 
    \end{array} \right]
\end{equation}
and
\begin{equation}
    \mW^2_O = \left[ \begin{array}{ccc}
        0 \\
        0 \\
        0 \\
        0 \\
        0 \\
        0 \\
        1 \\
        0 \\
        \end{array}\right], \quad
    \mW^2_V, \mW^2_K = \left[ \begin{array}{cccccc}
        0 & \dots & 0 & 1 & 0 & 0
    \end{array} \right], \quad
    \mW^2_Q = \left[ \begin{array}{cccccc}
        0 & 0 & 1 & 0 & \dots & 0
    \end{array} \right],
\end{equation}
where \(\dots\) denotes a sequence of zeros.

Define $\ff:\R^8\to \R^8$ of width size $q=5$ via:
\begin{align}
    \ff \left( \left[ 
    \begin{array}{ccc}
        x_1 \\
        x_2 \\
        x_3 \\
        x_4 \\
        x_5 \\
        x_6 \\
        x_7 \\
        x_8 \\
    \end{array} \right] \right) 
    & =
    \begin{bmatrix}
    0 & 0 & 0& 0& 0\\
    -M & 0 & 0& 0& 0\\
    0 & 0 & 0& 0& 0\\
    0 & -1 & 0 & 0& 0\\
    0 & 1 & \delta^{-d} - 1 & -1& 0\\
    0 & 0 & 0& 0& 0\\
    0 & 0 & 0& -1& 0\\
    0 & 0 & 0& 1& -1\\
    \end{bmatrix} \relu \left(
    \begin{bmatrix}
    0 &1 &0& -1 & 0 & 0 & 0& 0\\
    0 &0 &0& 1 & 0 & 0 & 0& 0\\
    0 &0 &0& 0 & 1 & 0 & 0& 0\\
    0 &0 &0& 0 & 0 & 0 & 1& 0\\
    0 &0 &0& 0 & 0 & 0 & 0& 1\\
    \end{bmatrix}
   \left[ 
    \begin{array}{ccc}
        x_1 \\
        x_2 \\
        x_3 \\
        x_4 \\
        x_5 \\
        x_6 \\
        x_7 \\
        x_8 \\
    \end{array} \right]
    + \left[ 
    \begin{array}{ccc}
        0 \\
        \epsilon \\
        0 \\
        0 \\
        0 \\
        0 \\
        0 \\
        0 \\
        0 \\
    \end{array} \right]\right) 
    \\
    & = \left[\begin{array}{ccc}
        0 \\
        - M \sigma_R(x_2-x_4+\epsilon)  \\
        0 \\
        - \sigma_R(x_4)  \\
        (\delta^{-d} - 1) \sigma_R(x_5) + \sigma_R(x_4) \\
        0 \\
        0 \\
        -\sigma_R(x_7) \\
        \sigma_R(x_7) - \sigma_R(x_8)  \\
    \end{array}\right],
\end{align}
where $0 < \epsilon < 1$ and $M > \frac{\delta^{-d}-1}{\epsilon}$. Then, we define two affine linear maps $\mathcal{L}_1 : \mathbb{R}^2 \to \mathbb{R}^{8}$ and $\mathcal{L}_2 : \mathbb{R}^{8} \to \mathbb{R}$ respectively:
\begin{equation}
    \mathcal{L}_1(x_1, x_2) \coloneqq (x_1, x_1, 1, 0, 0, x_2,0,0),\quad \mathcal{L}_2(x_1, x_2, x_3, x_4, x_5, x_6, x_7, x_8) \coloneqq 2 x_5 - x_6
\end{equation}
From \Cref{lemma:contextmap}, the corollary holds for this construction.
\end{proof}

\subsection{Step 3. Function Value Mapping with Bit Extraction}
\label{appendix:memorize}
We employ a bit extraction technique~\citep{bartlett1998almost} , as used \cite{zhang23} for weight-tied ReLU networks, to approximately memorize the label set. Given \( K \in \mathbb{N} \) input indices \( k = 1,2, \dotsc, K \) with associated values \( y_1, y_2, \dotsc, y_{K} \in \mathbb{R} \), the network approximately memorizes the differences \( y_i - y_{i-1} \) using their base-2 representation.  
Since the binary representation is limited to \( \{0,1\} \), the differences \( y_i - y_{i-1} \) must be rescaled by a factor $\epsilon \coloneq \max_i|y_{i} - y_{i-1}|$ as
\begin{equation}
    a_{i}= \big\lfloor \frac{y_{i}}{\epsilon} \big\rfloor,
\end{equation}
where $\lfloor x\rfloor \coloneq \max \{n: n\le x,\ n\in \Z\}$. Then, the difference \( b_i = a_i - a_{i-1} \) satisfies \( b_i \in \{-1, 0, 1\} \) and can be represented using two binary values \( c_i, d_i \in \{0,1\} \) as follows:
\begin{equation}
    b_{i}= c_{i} - d_{i},
\end{equation}
and we have 
\begin{equation}
    a_{k}= a_{0} + \sum_{i=0}^{k} b_{i} = a_{0} + \sum_{i=0}^{k} c_{i} - \sum_{i=0}^{k} d_{i} \quad \text{for} \quad k = 0, 1, \dotsc, n-1.
\end{equation}
\Cref{lemma:bitextraction} and \Cref{lemma:memorize} show that $\sum_{i=0}^{k} c_{i}$ and $\sum_{i=0}^{k} d_{i}$ can be realized by composition of single feed-forward layer. Thus, the network can approximate \( y_i \) using \( \epsilon a_i \), denoted as \( \tilde{y}_i \), with the following accuracy:
\begin{equation}
    |\tilde{y}_{i} - y_{i}| = |\epsilon \big\lfloor \frac{y_{i}}{\epsilon} \big\rfloor - \epsilon \frac{y_{i}}{\epsilon}| =  \epsilon | \big\lfloor \frac{y_{i}}{\epsilon} \big\rfloor - \frac{y_{i}}{\epsilon}| \leq \epsilon.
\end{equation}
For a $d$-dimensional input-output pair, we construct the networks for each dimension \ie
\begin{equation}
    \bm{\tilde{\vy}} = (\tilde{y}^1, \tilde{y}^2, \dotsc, \tilde{y}^d)
\end{equation}

The basic strategy of our lemma and proof follows Lemma D.1 from~\citet{zhang23}, as shown below and Proposition 3.2. However, their result cannot be directly applied here, as it requires depth-2 networks.

\begin{proposition}[Lemma D.1 in~\citet{zhang23}]
Given any \( r \in \mathbb{N}^+ \), there exists a function \( \ff: \mathbb{R}^{3d} \to \mathbb{R}^{3d} \) with width \( 8 \) and depth \( 2 \), utilizing two affine linear maps \( \mathcal{L}_1: \mathbb{R}^2 \to \mathbb{R}^5 \) and \( \mathcal{L}_2: \mathbb{R}^5 \to \mathbb{R} \), such that for any \( \theta_1, \theta_2, \dotsc, \theta_r \in \{0,1\} \), the following holds:

\begin{equation}
    \calL_2\circ\ff^{\circ r}\circ \calL_1\big(k,\ \bin 0.\theta_1\theta_2\cdots\theta_{r}\big)=\sum_{\ell=1}^{k}\theta_\ell \quad \tn{for $k=1,2,\dotsc,r$,}
\end{equation}
where $\bin 0.\theta_1\theta_2\cdots \theta_r$ denote the binary representation of $\theta= \sum_{l=1}^{r}\theta_l 2^{-l}$.
\end{proposition}

We found that the \textit{loop unrolling} technique allows us to reduce the number of layers from $2$ to $1$ by replacing $x^{k+1} = \ReLU(\ReLU(x'^k))$ with $(x^{k+1}, x'^k) = \ReLU(x'^k, x^k)$.
Although our method makes the weights dependent on $\theta_1, \theta_2, \dots, \theta_{r}$, this does not present an issue for our construction in function approximation. Specifically, ${\theta_1, \theta_2, \dots, \theta_{r}}$ is fixed for each target function, and the role of the network is to learn the weights tailored to that single function.

\begin{lemma}\label{lemma:bitextraction}
Given $\theta_1,\theta_2,\dotsc,\theta_{r}\in \{0,1\}$ for some $r\in \N^+$, there exists a feed-forward layer $\ff: \R^6 \to \R^6$ of width size $9$ and two affine linear maps $\mathcal{L}_1:\R\to\R^6$ and $\mathcal{L}_2:\R^6\to \R$ \st
\begin{equation}
\mathcal{L}_2\circ {(\mathrm{id} + \ff)}^{\circ r}\circ \mathcal{L}_1\big(k) = \sum_{l=1}^{k}\theta_l \quad \text{for } k=1,2,\dotsc,r,
\end{equation}
where the bit complexity is bounded by $\mathcal{O}(r)$. 
\end{lemma}
\begin{proof}
From \cref{lemma:step}, we have a function $\calT_\epsilon(x)$, for $\epsilon > 0$, defined by
\begin{equation}
    \calT_\epsilon(x) = \relu\big(\tfrac{x}{\epsilon}+1\big)-\relu\big(\tfrac{x}{\epsilon}\big),
\end{equation}
as shown in \cref{fig:step}, and it satisfies
\begin{equation}
    \calT_\epsilon(x) = 
    \begin{cases}
        1 & \text{if } x \geq 0,\\
        0 & \text{if } x \leq 0 - \epsilon.
    \end{cases}
\end{equation}
\begin{figure}[ht]
\begin{center}
    \includegraphics[width=0.8\linewidth]{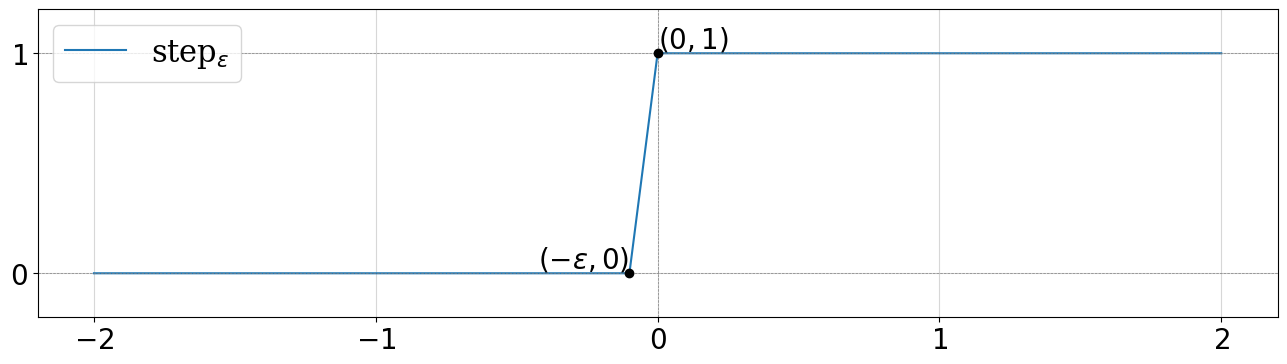}
\end{center}
\caption{An illustration of $\calT_\epsilon(x)$.}\label{fig:step}
\end{figure}

Define $\beta_l$ for $l=0,1,\dotsc,r$ as
\begin{equation}
	\beta_l=\bin 0.\theta_{l}\cdots \theta_r,
\end{equation}
where $\bin 0.\theta_l\cdots \theta_r$ denote the binary representation of $\theta= \sum_{i=l}^{r}\theta_i 2^{-i}$ and $\beta_0 \coloneq 0$. If we set $\epsilon < 2^{-r}$, it follows that
\begin{align}
\theta_l & = \calT_\epsilon(\bin 0.\theta_{l}\cdots \theta_r-\tfrac{1}{2}) \\
 & = \calT_\epsilon(\beta_l-\tfrac{1}{2}),
\end{align}
implying, for $l=0,1,\dotsc,r-1$,
\begin{align}
\beta_{l+1} & = 2\beta_l - \theta_l \\ 
&= 2\beta_l- \calT_\epsilon\big(\beta_l-\tfrac{1}{2}\big).
\end{align}
For all \( l = 1, \dotsc, r \), since the product \( xy \) satisfies  
\begin{align}
xy &= \max \{ 0, x + y - 1 \}\\ &= \relu(x + y - 1),
\end{align}
for \( x, y \in \{0,1\} \), it follows that
\begin{align}
    \sum_{l=1}^k \theta_l &= 
    \sum_{l=1}^k \theta_l + \sum_{l=k+1}^r 0 \\
    &= \sum_{l=1}^{r}\theta_l\cdot\calT_\epsilon(k-l) \\
    &= \sum_{l=1}^{r}\relu\Big(\theta_l+\calT_\epsilon(k-l)-1\Big)\\
    &= \sum_{l=1}^{r}\relu\Big(\calT_\epsilon(\beta_l-\tfrac{1}{2})+\calT_\epsilon(k-l)-1\Big).
\end{align}
To compute the right-hand side, we require two nested $\ReLU$ functions. By employing loop unrolling, we precompute $\mathcal{T}(\beta_l - \frac{1}{2})$ and $\mathcal{T}(k-l)$ in the previous iterations, reducing the requirement to a single layer.  

Define $\boldsymbol{\xi}_l$ for $l = 0, 1, \dotsc, r-1$ as  
\begin{equation}
\begin{split}
  \boldsymbol\xi_l &= {\Big(k - l,\  \beta_{l},\  \beta_{l+1},\  \calT_\epsilon(\beta_l - \frac{1}{2}),\  \calT_\epsilon(k-l),\  \mathrm{sum}(l)\Big)}^\top, \\
  &\text{where} \quad \mathrm{sum}(l) \coloneq \sum_{i=1}^{l}\relu\Big(\calT_\epsilon(\beta_i-\tfrac{1}{2})+\calT_\epsilon(k-i)-1\Big).
\end{split}
\end{equation}
Note that we have $\beta_{l+1}$ in the $l$-th loop to precompute $\calT_\epsilon(\beta_{l+1} - \frac{1}{2})$ and $\calT_\epsilon \big((k-(l+1)\big)$ for the $(l+1)$-th loop.

Define $\ff:\R^6\to\R^6$ with a width size of $9$ such that
\begin{align}
&\big(\mathrm{id} + \ff\big) (\boldsymbol\xi_l)
= \big(\mathrm{id} + \ff\big)
\left(
\left[
\begin{array}{c}
k - l  \\
\beta_{l} \\
\beta_{l+1}  \\
\calT_\epsilon(\beta_{l} - \frac{1}{2})  \\
\calT_\epsilon (k-l)  \\
\mathrm{sum}(l) \\
\end{array}
\right]
\right)\\
&=
\left[
\begin{array}{c}
k - l  \\
\beta_{l} \\
\beta_{l+1}  \\
\calT_\epsilon(\beta_{l} - \frac{1}{2})  \\
\calT_\epsilon (k-l)  \\
\mathrm{sum}(l) \\
\end{array}
\right] + 
\begin{bmatrix}
0 & 0 & 0 & 0 & 0 & 0& 0 & 0 & 0 \\
1 & -1 & 0 & 0 & 0 & 0& 0 & 0 & 0 \\
0 & 0 & 1 & -1 & 1 & 0& 0 & 0 & 0 \\
0 & -1 & 0 & 1 & -1 & -1& 0 & 0 & 0 \\
0 & 0 & 0 & 0 & 0 & -1 & -1 & 1 & 0 \\
0 & 0 & 0 & 0 & 0 & 0& 0 & 0 & 1 \\
\end{bmatrix} \notag \\
&\relu \left(
\begin{bmatrix}
0 & 1 & 0 & 0 & 0 & 0 \\
0 & 0 & 0 & 1 & 0 & 0 \\
0 & 0 & 1 & 0 & 0 & 0 \\
0 & 0 & 1/\epsilon & 0 & 0 & 0 \\
0 & 0 & 1/\epsilon & 0 & 0 & 0 \\
0 & 0 & 0 & 0 & 1 & 0 \\
1/\epsilon & 0 & 0 & 0 & 0 & 0 \\
1/\epsilon & 0 & 0 & 0 & 0 & 0 \\
0 & 0 & 0 & 1 & 1 & 0 \\
\end{bmatrix}
\left[\begin{array}{c}
k - l  \\
\beta_{l} \\
\beta_{l+1}  \\
\calT_\epsilon(\beta_{l} - \frac{1}{2})  \\
\calT_\epsilon (k-l)  \\
\mathrm{sum}(l) \\
\end{array}\right]
+\left[\begin{array}{c}
0  \\
0  \\
0  \\
-1/(2\epsilon) + 1 \\
-1/(2\epsilon)  \\
0 \\
-1/\epsilon + 1\\
-1/\epsilon  \\
-1  \\
\end{array}\right]
\right)
+
\left[\begin{array}{c}
-1  \\
0  \\
0  \\
0  \\
0  \\
0  \\
\end{array}
\right]
\\
&=
\left[\begin{array}{c}
k - l  \\
\beta_{l} \\
\beta_{l+1}  \\
\calT_\epsilon(\beta_{l} - \frac{1}{2})  \\
\calT_\epsilon (k-l)  \\
\mathrm{sum}(l) \\
\end{array}\right]
+
\left[
\begin{array}{ccc}
-1 \\
\relu(\beta_{l}) - \relu\big(\calT_\epsilon(\beta_{l} - \frac{1}{2})\big) \\
\relu(\beta_{l+1}) - \big(\relu(\tfrac{\beta_{l+1}-1/2}{\epsilon}+1) -\relu(\tfrac{\beta_{l+1}-1/2}{\epsilon})\big) \\
-\relu\big(\calT_\epsilon(\beta_{l} - \frac{1}{2})\big) + \relu(\tfrac{\beta_{l+1}-1/2}{\epsilon}+1) -\relu(\tfrac{\beta_{l+1}-1/2}{\epsilon}) \\
-\relu\big(\calT_\epsilon(k-l)\big) + \relu(\tfrac{k-(l+1)}{\epsilon}+1) -\relu(\tfrac{k-(l+1)}{\epsilon}) \\
\relu\Big( \calT_\epsilon(k-l) + \calT_\epsilon(\beta_l - \frac{1}{2}) - 1\Big)\\
\end{array}
\right]\\
&=
\left[
\begin{array}{ccc}
k - (l+1)  \\
2\beta_{l} - \calT_\epsilon(\beta_{l} - \frac{1}{2}) \\
2\beta_{l+1} - \calT_\epsilon(\beta_{l+1} - \frac{1}{2})  \\
\calT_\epsilon(\beta_{l+1} - \frac{1}{2})  \\
\calT_\epsilon \big((k-(l+1)\big)  \\
\mathrm{sum}(l+1) \\
\end{array}
\right] =
\left[
\begin{array}{ccc}
k - (l+1)  \\
\beta_{l+1} \\
\beta_{l+2}  \\
\calT_\epsilon(\beta_{l+1} - \frac{1}{2})  \\
\calT_\epsilon \big((k-(l+1)\big)  \\
\mathrm{sum}(l+1) \\
\end{array}
\right] = \boldsymbol\xi_{l+1},
\end{align}
Define $\mathcal{L}_1:\R \to\R^6$ and $\calL_2:\R^6\to\R$ via 
\begin{equation}
    \mathcal{L}_1(k)\coloneqq{(k, \beta_{0},\beta_{1},0,0,0)}^\top=\boldsymbol\xi_0, \quad \mathcal{L}_2(x_1,x_2, x_3, x_4, x_5, x_6)\coloneqq x_6,
\end{equation}
respectively. The lemma holds for this construction.
\end{proof}
Then, we prove \Cref{lemma:memorize} with \Cref{lemma:bitextraction}.

\memorize*

\begin{proof}

We prove this for the case where $d=1$, considering $y_k \in \R$ for $k=0,1, \dotsc, K-1$. Define
\begin{equation}
    a_{k}= \big\lfloor \tfrac{y_{k}}{\varepsilon} \big\rfloor 
    \quad \tn{for $k=0,1,\dotsc,K-1$},
\end{equation}
where $\lfloor x\rfloor=\max \{n: n\le x,\ n\in \Z\}$
and set
\begin{equation}
    b_{k}=a_{k}-a_{k-1}\quad \tn{for $k=1,2,\dotsc,K-1$.}
\end{equation}
Since $b_{k}\in \{-1,0,1\}$, there exist $c_{k}\in\{0,1\}$ and $d_{k}\in\{0,1\}$ such that 
\begin{equation}
    b_{k}=c_{k}-d_{k}\quad \tn{ for $k=1,2,\dotsc,K-1$.}
\end{equation}
Thus, we have
\begin{equation}
    \begin{split}
        a_{k}=a_0+\sum_{i=1}^{ k }c_{i}
        -\sum_{i=1}^{ k }d_{i}\quad \text{for any } k \in \{1,2,\dotsc,K-1\}
    \end{split}
\end{equation}

From \Cref{lemma:bitextraction}, there exist $\ff^{(c)}, \ff^{(d)}: \R^6 \to \R^6$ of width
size $9$ and affine linear maps $\mathcal{L}'_2: \R^6 \to \R$ and $\mathcal{L}^{(c)}_1, \mathcal{L}^{(d)}_1: \R \to \R^6$ \st
\begin{align}
    \mathcal{L}'_2\circ {(\mathrm{id} + \ff^{(c)})}^{\circ (m-1)}\circ \mathcal{L}^{(c)}_1\big(k) = \sum_{i=1}^{k}c_i, \quad
    \mathcal{L}'_2\circ {(\mathrm{id} + \ff^{(d)})}^{\circ (m-1)}\circ \mathcal{L}^{(d)}_1\big(k) = \sum_{i=1}^{k}d_i,
\end{align}
for $k=0,1,\dotsc,K-1$. Then, define $\ff: \R^{12}\to\R^{12}$ of width size $18$, for $\vx, \vy \in \R^6$,
\begin{equation}
    \ff\left( \left[
    \begin{array}{ccc}
        \vx  \\
        \vy  \\
    \end{array}\right] \right)\coloneqq 
    \left[
    \begin{array}{ccc}
        \ff^{(c)}(\vx)  \\
        \ff^{(d)}(\vy)  \\
    \end{array}\right].
\end{equation}
Define $\mathcal{L}_1:\R\to\R^{12}$ and $\mathcal{L}_2:\R^{12}\to\R$ as
\begin{equation}
\mathcal{L}_1(x)\coloneqq
 \left[
    \begin{array}{ccc}
        \mathcal{L}^{(c)}_1(x)  \\
        \mathcal{L}^{(d)}_1(x)  \\
    \end{array}\right]
, \quad \mathcal{L}_2\left( \left[
    \begin{array}{ccc}
        \vx  \\
        \vy  \\
    \end{array}\right] \right)
    \coloneqq \eps \big(a_0+\mathcal{L}'_2(\vx)-\mathcal{L}'_2(\vy)\big).
\end{equation}
We can confirm that 
\begin{align}
    &\calL_2\circ (\mathrm{id} + \mathrm{FF})^{\circ(K-1)} \circ \calL_1(k)\\
    &= \calL_2\circ (\mathrm{id} + \mathrm{FF})^{\circ(K-1)} \left( \left[
    \begin{array}{ccc}
        \mathcal{L}^{(c)}_1(k)  \\
        \mathcal{L}^{(d)}_1(k)  \\
    \end{array}\right] \right)\\
    &= \calL_2 \left( \left[
    \begin{array}{ccc}
        (\mathrm{id} + \mathrm{FF}^{(c)})^{\circ(K-1)} \circ \mathcal{L}^{(c)}_1(k)  \\
        (\mathrm{id} + \mathrm{FF}^{(d)})^{\circ(K-1)} \circ\mathcal{L}^{(d)}_1(k)  \\
    \end{array}\right] \right) \\
    &= \eps \big(a_0 +\sum_{i=1}^{k}c_i - \sum_{i=1}^{k}d_i\big)= \eps a_k.
\end{align}
Thus we have
\begin{equation}
    |\calL_2\circ (\mathrm{id} + \mathrm{FF})^{\circ(K-1)} \circ \calL_1(k) - y_k| = |\eps a_{k} - y_{k}| \leq \varepsilon.
\end{equation}
For $d$-dimensional inputs, we need $d$-times more parameters.
\end{proof}

\newpage

\section{Role of Time-dependent Scaling Parameters}
\label{appendix:time_dependent}

We show that time-dependent scaling parameters overcome the limitations inherent to the looped architecture and eliminate the dependence of the modulus of continuity. We use the architecture defined in \Cref{sec:lim} as:
\begin{equation}
    \ff(\vx) \to \boldsymbol\eta(t) \odot \ff(\vx) \quad \text{for the $t$-th loops},
\end{equation}
The following lemma demonstrates that time-dependent scaling parameters can exactly map indices to output vectors.

\timedependentmemorize*

\begin{proof}
We consider the case when $d=1$, where $y_k \in \R$ for $k=0, 1, \dotsc, K-1$. We update $y_k$ as follows:
\begin{equation}
    y_k \to y_k + \epsilon,
\end{equation}
where $\epsilon$ is chosen such that none of the $y_l$ values are zero.

Next, we define $\eta(l) \in \R^4$ as:
\begin{equation}
    \eta(l) \coloneq (0, 1, \frac{y_{l}}{y_{l-1}} - 1, \frac{y_{l}}{y_{l-1}})^\top \quad \text{for } l = 1, 2, \dotsc, K-1.
\end{equation}

By \Cref{lemma:impulse}, we have, $x \in [-M,M]$ and $t \in \N$,
\begin{align}
    \mathrm{impulse}_0(x,t) &= \relu\big(x+2M(t+1/2)\big)-2M\relu(t+1/2) \notag \\ & \qquad -\relu\big(x+2M(t-1/2)\big)+2M\relu(t-1/2)\\
    &=\begin{cases}
        x \quad & \tn{if}\  t = 0,  \\
        0 \quad  & \tn{otherwise},
    \end{cases}
\end{align}
where $M>\max_{k \in \{0,1,\dotsc,K-1\}} y_k$.

Let \(k \in \{0,1,\dotsc,K-1\}\) be the input index that specifies which \(y_k\) to extract.
Define 
\begin{equation}
     s(l) \coloneq \sum_{i=0}^{l}\mathrm{impulse}_0\big(y_{i}, k - i\big),
\end{equation}
for $l=0,1,\dotsc,K-1$, which satisfies
\begin{equation}
    s(K-1) = y_k.
\end{equation}

Define $\boldsymbol\xi_l \in \R^4$ via
\begin{equation}
  \boldsymbol\xi_l \coloneq \Big(k,\  k-l-1,\  y_l,\  s(l))\Big)^\top.
\end{equation}
for $l=0,1,\dotsc,K-1$.

Then, define $\ff:\R^4\to \R^4$ of width size $q=6$ via:
\begin{align}
    & (\mathrm{id} + \eta(l)\odot\mathrm{FF})(\boldsymbol\xi_{l-1})
    =
    \boldsymbol\xi_{l-1}+ \notag \\
    &\eta(l)\odot\left(\begin{bmatrix}
    0 & 0 & 0 & 0 & 0 & 0 \\
    0 & 0 & 0 & 0 & 0 & 0 \\
    1 & -1 & 0 & 0 & 0 & 0 \\
    0 & 0 & 1 & -1 & -2M & 2M \\
    \end{bmatrix} \relu \left(
    \begin{bmatrix}
    0 & 0 & 1 & 0 \\
    0 & 0 & -1 & 0 \\
    0 & 2M & 1 & 0 \\
    0 & 2M & 1 & 0 \\
    0 & 1 & 0 & 0 \\
    0 & 1 & 0 & 0 \\
    \end{bmatrix}
   \boldsymbol\xi_{l-1}
    + \left[ 
    \begin{array}{ccc}
        0 \\
        -1 \\
        M \\
        -M \\
        1/2 \\
        -1/2 \\
    \end{array} \right]    
    \right)
    + \left[ 
    \begin{array}{ccc}
        0 \\
        -1 \\
        0 \\
        0 \\
    \end{array} \right] \right)
    \\
    & = \boldsymbol\xi_{l-1}
    + \left[\begin{array}{ccc}
        0 \\
        1 \\
        \tfrac{y_{l}}{y_{l-1}} - 1 \\
        \tfrac{y_{l}}{y_{l-1}} \\ 
    \end{array}\right]
    \odot\left[\begin{array}{ccc}
    0 \\
    -1 \\
    \relu(y_{l-1}) - \relu(-y_{l-1}) \\
    \Bigg(\relu\big(y_{l-1}+2M((k-l)+1/2)\big)\\
    -2M\relu((k-l)+1/2) \\ 
    \relax 
    -\relu\big(y_{l-1}+2M(k-l-1/2)\big) + 2M\relu(k-l-1/2)\Bigg) \\
\end{array}\right] \\
    & = \left[\begin{array}{ccc}
        k \\
        k-l \\
        y_{l-1} \\
        s(l-1) \\
    \end{array}\right]
    + \left[\begin{array}{ccc}
        0 \\
        1 \\
        \tfrac{y_{l}}{y_{l-1}} - 1 \\
        \tfrac{y_{l}}{y_{l-1}} \\ 
    \end{array}\right]
    \odot\left[\begin{array}{ccc}
        0 \\
        -1 \\
        y_{l-1} \\
        \mathrm{impulse}_0\big(y_{(l-1)}, k-l)\big) \\ 
    \end{array}\right] \\
    & = \left[\begin{array}{ccc}
        k \\
        k-l-1 \\
        y_{l} \\
        s(l) \\
    \end{array}\right] = \boldsymbol\xi_l.
\end{align}
for $l=1,2,\dotsc,K-1$. Thus we have
\begin{equation}
    (\mathrm{id} + \eta(K-1)\odot\mathrm{FF})\circ \cdots \circ (\mathrm{id} + \eta(1)\odot\mathrm{FF}) (\boldsymbol\xi_0) = \boldsymbol\xi_{K-1}
\end{equation}

Then, define two affine linear maps $\mathcal{L}_1 : \mathbb{R} \to \mathbb{R}^{4}$ and $\mathcal{L}_2 : \mathbb{R}^{4} \to \mathbb{R}$ by
\begin{equation}
    \mathcal{L}_1(x) \coloneqq (k, k, y_0, 0), \quad
    \mathcal{L}_2(x_1, x_2, x_3, x_4) \coloneqq x_4-\epsilon.
\end{equation}

We can extend this to $d$-dimensional input by using $d$ times more parameters, by applying the above to each dimension
\end{proof}

\section{Details of Experiments}
\label{app:details_of_exp}
This appendix section provides additional details on the experiments for each task.

\subsection{Reasoning Tasks}

\subsubsection{Problem Settings}\label{app:setting}

\textbf{Longest Common Subsequence (LCS)} is the longest common to a given set of sequences. We use problems with input lengths of $60$ and $100$. Two sequences are sampled uniformly from the alphabet.

\textbf{Edit Distance (ED) problem}, also known as Levenshtein distance, is to find the minimum cost required to change one sequence into the other.
We adopted the problem setting and data generation approach from~\citet{feng2023towards}, but applied larger input lengths. The costs for insertion, deletion, and replacement were set to $2$, $2$, and $3$, respectively. 
They generate instances of the edit distance problem as shown in Algorithm \ref{alg:ED}. The first string is randomly selected, while the second is generated in two ways: (1) a random string yielding a large edit distance, and (2) a corrupted copy of the first string, resulting in a small edit distance. 
%
\begin{algorithm}[ht]
\caption{ED Data Generation from \citet{feng2023towards}}
\label{alg:ED}
\begin{algorithmic}[1]
\STATE {\bfseries Input:} Length of the First String $n$
\STATE {\bfseries Input:} Alphabet $V=\{a,b...z \}$
\STATE {\bfseries Output:} Sequence $s_1$, $s_2$
Sample $t$ uniformly from $\{3,4...10\}$ \;
$T$ $\leftarrow$ Sample $t$ letters from $V$ \;
$s_1$ $\leftarrow$ Sample $n$ letters uniformly from $T$ \;
Sample $p$ uniformly from $[0,1]$ \; 
\IF{$p < 0.4$}
\STATE Sample $l$ uniformly from $\{n-3,n-2,...,n+2\}$\
\STATE $s_2$ $\leftarrow$ Sample $l$ letters uniformly from $T$
\ELSE
\WHILE{$len(s_2)$ not in $[n-3, n+2]$}
\STATE $s_2 \leftarrow s_1$ \; 
\FOR{$i\leftarrow 1$ to $n$}
\STATE Sample $p$ uniformly from $\{0,1...len(s_2)-1\}$\;
\STATE Sample $l$ uniformly from $T$\;
\STATE Randomly conduct one of the followings: pop $s_2[p]$, substitute $s_2[p]$ with $l$, insert $l$ into $s_2[p]$\;
\ENDFOR
\ENDWHILE
\ENDIF
\end{algorithmic}
\end{algorithm}

\paragraph{Sudoku}
We use the dataset from \citet{radcliffe2020sudoku}, which contains over 3 million Sudoku puzzles. The puzzles are flattened, with $0$ representing blank grids. The input sequence is formatted as:
\begin{center}
\texttt{100503700603008090000009800010000000876100000000006000000000007080907604700060312}.
\end{center}

\paragraph{Countdown}
To generate the dataset, we randomly sampled the target and input numbers for each instance. Pairs that have no solution were excluded. For tokenization, we assigned unique tokens to each number and symbol. The target sequence is represented as:
\begin{center}
\texttt{58 84 48 62 96 62 - 58 = 4 48 / 4 = 12 84 + 12 = 96 }.
\end{center}
The model learns to predict the target sequence from the input:
\begin{center}
\texttt{58 84 48 62 96 0 0 0 0 0 0 0 0 0 0 0 0 0 0 0 }.
\end{center}

\subsubsection{Training and Test Accuracy Correlation for Edit Distance}\label{app:ED}
Given that our study focuses on function approximation capabilities, one might question whether it is appropriate to rely on test evaluations, which are influenced by generalization. Here, we confirm that there is a strong correlation between training and test results, validating this approach. \Cref{fig:edit} demonstrates a strong positive correlation between training and test accuracy, enabling the evaluation of approximation power through test accuracy.

\begin{figure}[ht]
\begin{center}
    \begin{minipage}{0.46\linewidth}
        \centering
        \includegraphics[width=\linewidth]{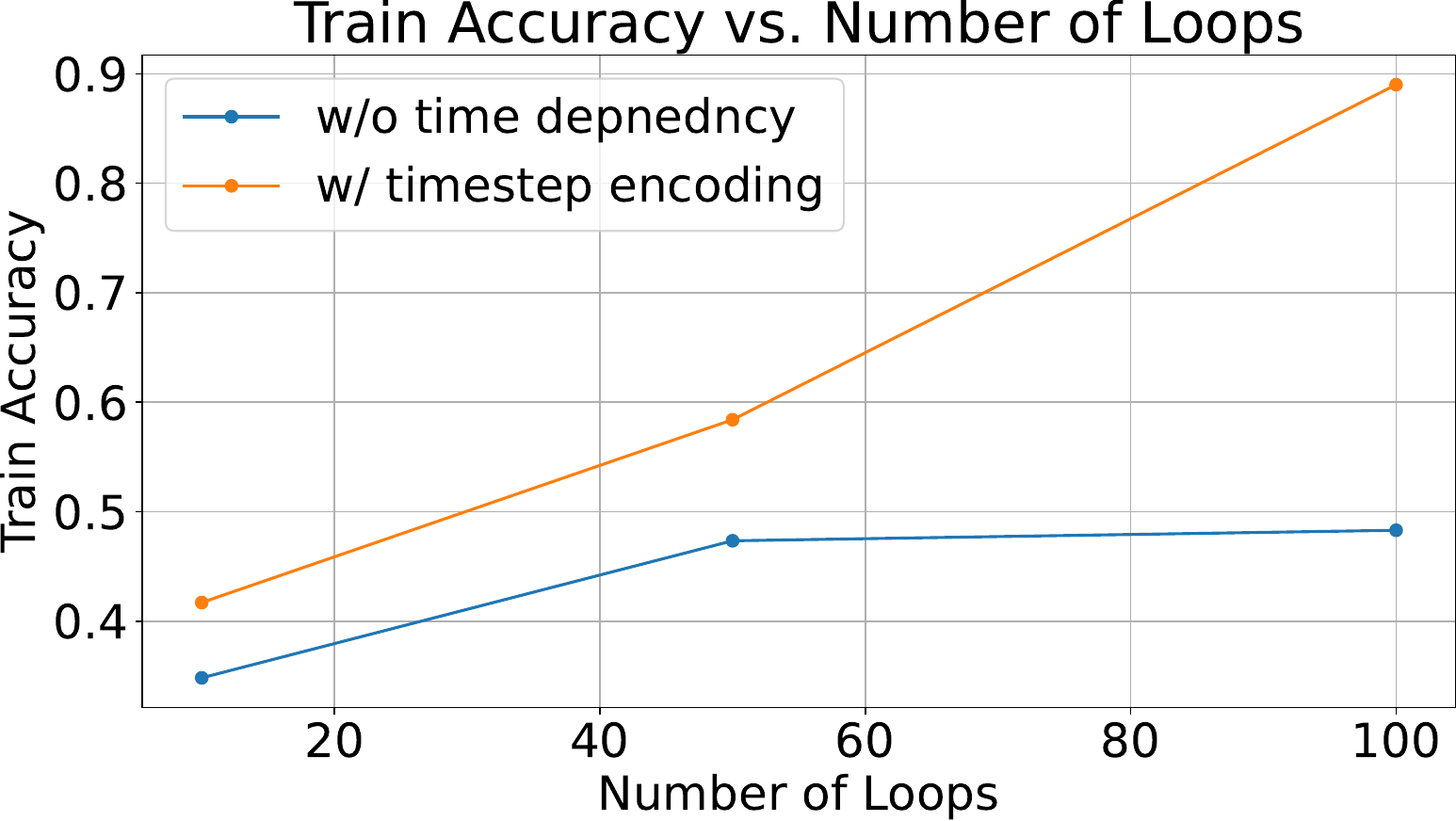}
    \end{minipage}
    \hspace{0.05\linewidth}
    \begin{minipage}{0.46\linewidth}
        \centering
        \includegraphics[width=\linewidth]{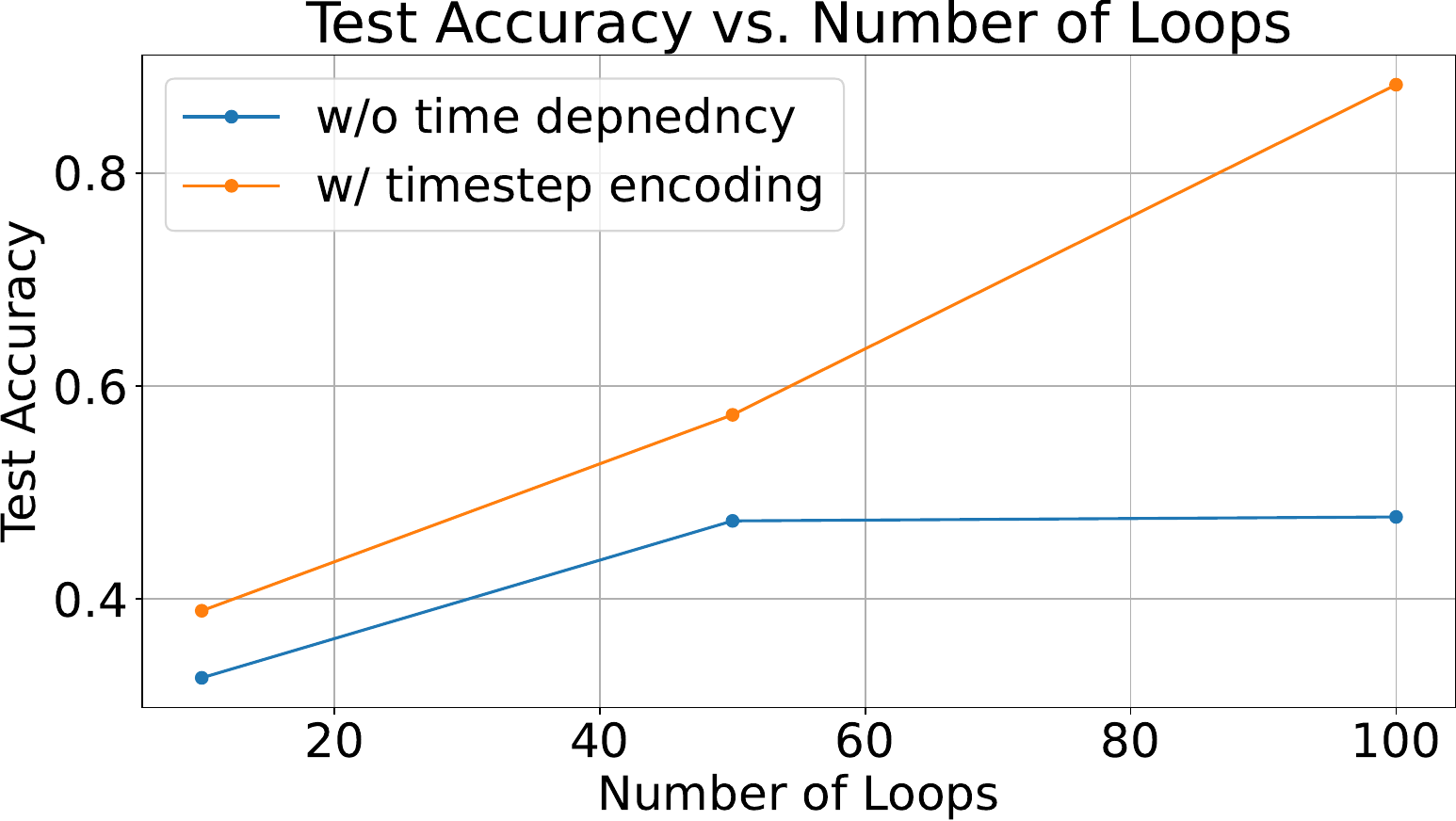}
    \end{minipage}
    \caption{Training and test accuracy for the edit distance task with a sequence length of $60$.}
    \label{fig:edit}
\end{center}
\end{figure}

\subsubsection{Model and Training Configuration}\label{app:conf}
We used Looped Transformers with $4$ attention heads and a $256$-dimensional embedding. The AdamW optimizer \citep{loshchilov2018fixing} was used with $\beta_1 = 0.9$, $\beta_2 = 0.999$, a weight decay of $0.01$, and a linear learning rate decay scheduler starting at $\text{lr} = 10^{-4}$ and ending at $0$, with $5$ warm-up steps. Training consisted of $50$ epochs for reasoning, $200$K steps for in-context learning, and $100$K iterations for language modeling, using a batch size of $64$. For time-dependent models, $\boldsymbol{\gamma}(t)$ and $\boldsymbol{\alpha}(t)$ were initialized as zero and one vectors, respectively, following \citet{peebles2023scalable}. The input embeddings are added at each loop iteration. Furthermore, for Sudoku and in-context learning tasks, the output of each intermediate loop is incorporated into the loss as~\cite{yang2023learning}.

\subsection{In-Context Learning}
\label{app:incontext}
We followed the setting of~\citet{yang2024looped}. The problem is to learn the function class from a given sequence composed of the pairs of input $\vx_i$ and output values $f(\vx_i)$. The input for model is \((\vx_1, f(\vx_1), \dots, \vx_k, f(\vx_k), \vx_{\text{test}})\), and model learns to predict \(f(\vx_{\text{test}})\).
The model is trained on $f(\vx_k)$ and its performance is evaluated on \(f(\vx_{\text{test}})\) using the squared error.

We use depth-4 decision trees with 20-dimensional inputs. Each function in this class is represented by a full binary tree with 16 leaf nodes. Non-leaf nodes are associated with specific input coordinates, while leaf nodes are assigned target values. To evaluate f($\vx$), the tree is traversed from the root, moving to the right if the coordinate value is positive and to the left otherwise. Inputs and leaf node values are sampled from N(0,$\mI$), and the coordinates for non-leaf nodes are drawn uniformly at random.
Our training setup follows the approach of~\citet{yang2024looped}. Following the curriculum training approach of~\citet{garg2022what,yang2024looped}, we progressively increase the task dimensionality from $5$ to $20$ in steps of $1$ every $5000$ steps, while the sequence length increases from $26$ to $101$ in increments of $5$ over the same interval.

\subsection{Language Modeling}
\label{app:langugae}
Tokenization is performed using byte-pair encoding, following GPT-2~\cite{Radford2019}. The Transformer model is based on the GPT-2 decoder architecture~\citep{Radford2019}. The baseline standard Transformer model consists of $6$ layers, $8$ attention heads, and an embedding size of $512$. The Looped Transformer has a $1$ layer, $12$ attention heads, and a hidden dimension of $768$, which were chosen to match the parameter size of the baseline. 
We initialize $\boldsymbol\gamma(t)$ as zero vectors and $\boldsymbol\alpha(t)$ as one vector for time-dependent models. The AdamW optimizer~\citep{loshchilov2018fixing} is used with $\beta_1 = 0.9$, $\beta_2 = 0.95$, a weight decay of $1 \times 10^{-4}$, and a learning rate schedule with $2000$ warmup steps. The maximum learning rate is set to $2 \times 10^{-4}$ and decays to $6 \times 10^{-5}$ using a cosine schedule.
Training is conducted for $100\text{k}$ iterations with a batch size of $48$ and a block size of $1024$.

\section{Disccusion}

\paragraph{Multiple Layers}

A natural question is whether our analysis, which focuses on single-layer Looped Transformers, can be extended to multi-layer architectures.
We restricted our analysis to a single layer in order to highlight a key strength of Looped Transformers—namely, their universality as function approximators even with just one layer.
A more specific question is whether multi-layer Looped Transformers can overcome potential limitations inherent to the single-layer design. While it is conceivable that deeper architectures but with fixed-depth feedforward layers may achieve better approximation accuracy, this remains an open question. The difficulty lies in the fact that such improvements are not captured in terms of asymptotic order but rather in constants, which are harder to analyze theoretically.
For instance, if we allow a logarithmic number of layers depending on the desired approximation precision, then even our current construction may overcome the limitations of the looped architecture. However, this deviates from our main objective, which is to characterize the approximation rate solely in terms of the number of loops.

\paragraph{Additional Experiments}
To assess the model's sensitivity to input continuity, we designed a perturbed version of the WikiText-103 dataset, where $10\%$ of the tokens were randomly replaced. We trained models with and without timestep encoding and evaluated both their memorization performance and continuity behavior. 
Continuity was measured by applying small perturbations to the input and quantifying the change in output embeddings. The model with timestep encoding showed improved memorization (cross-entropy loss reduced from $4.32$ to $4.18$) and a significant reduction in continuity coefficients (from $130.6$ to $21.5$). 
These results suggest that timestep encoding not only enhances stability under perturbations but also enables more faithful input-output mappings, thereby improving both robustness and learning efficiency.







\end{document}